%% file: _main.tex
\documentclass{article}

\usepackage[nonatbib, final]{neurips_2022}

\usepackage[utf8]{inputenc} 
\usepackage[T1]{fontenc}    
\usepackage{hyperref}       
\usepackage{url}            
\usepackage{booktabs}       
\usepackage{amsfonts}       
\usepackage{nicefrac}       
\usepackage{microtype}      
\usepackage{xcolor}         

\usepackage{lipsum}		    
\usepackage{graphicx}
\usepackage{doi}

\usepackage{todonotes}
\usepackage[]{caption}
\usepackage[subrefformat=parens]{subcaption}
\usepackage{wrapfig}
\usepackage{here}
\usepackage{mathtools} 
\usepackage{url}            
\usepackage{nicefrac}       
\usepackage{microtype}      
\usepackage{multirow}
\usepackage{amsmath, amssymb} 
\usepackage{type1cm} 
\usepackage{bm} 
\usepackage{bbm} 
\usepackage{dsfont} 
\usepackage{cases} 
\usepackage{amsthm} 
    \newtheorem{theorem}{Theorem}[section]
    \newtheorem{corollary}{Corollary}[section]
    \newtheorem{lemma}{Lemma}[section]

\usepackage{multibib}
\newcites{App}{Supplementary references}

\def\nn{\nonumber \\ }

\def\ti{\tilde}
\def\mr{\mathrm} 
\def\f{\frac}

\def\mbr{\mathbb{R}} 

\def\mbz{\mathbb{Z}}
\def\mbn{\mathbb{N}}

\def\mcd{\mathcal{D}}

\def\mca{\mathcal{A}}

\def\mbone{\mathds{1}}
\def\al{\alpha}
\def\be{\beta}
\def\ga{\gamma}

\def\ep{\epsilon}

\def\et{\eta}
\def\th{\theta}

\def\la{\lambda}

\def\ta{\tau}

\def\De{\Delta}

\def\La{\Lambda}

\def\bth{{\bm \th}}
\def\dbth{\dot{\bm \th}}
\def\ddbth{\ddot{\bm \th}}
\def\dddbth{\dddot{\bm \th}}
\def\btht{{\bm \th} (t)}
\def\dbtht{\dot{\bm \th} (t)}
\def\ddbtht{\ddot{\bm \th} (t)}

\def\bthketa{\bm{\theta}(k\eta)}

\def\berr{\bm e}
\def\bpsi{\bm \psi}
\def\bgbtht{\bm{G} (\btht, \al)}
\def\bg{\bm{g}}

\def\bG{\bm{G}}
\def\bx{\bm{\xi}}
\def\tbx{\tilde{\bm{\xi}}}
\def\bthperp{\bm{\theta}_{\mca\perp}}
\def\bthpara{\bm{\theta}_{\mca\parallel}}
\def\bthperpt{\bm{\theta}_{\mca\perp} (t)}
\def\bthparat{\bm{\theta}_{\mca\parallel} (t)}

\def\dbthpara{\dot{\bm{\theta}}_{\mca\parallel}}

\def\dbthparat{\dot{\bm{\theta}}_{\mca\parallel} (t)}

\def\bthaperp{\bm{\theta}_{\mca\perp}}
\def\bthapara{\bm{\theta}_{\mca\parallel}}

\def\dbthaperp{\dot{\bm{\theta}}_{\mca\perp}}
\def\dbthapara{\dot{\bm{\theta}}_{\mca\parallel}}
\def\dbthaperpt{\dot{\bm{\theta}}_{\mca\perp} (t)}

\def\hbth{\hat{\bm{\theta}}}

\def\dhbth{\dot{\hat{\bm{\theta}}}}

\def\mcac{{\mathcal{A}^\mathsf{c}}}
\def\btha{\bm{\theta}_{\mathcal{A}}}
\def\bthac{\bm{\theta}_{\mathcal{A}^{\mathsf{c}}}}
\def\bthat{\bm{\theta}_{\mathcal{A}}(t)}
\def\bthact{\bm{\theta}_{\mathcal{A}^{\mathsf{c}}}(t)}
\def\hbtha{\hat{\bm{\theta}}_{\mathcal{A}}}
\def\hbthat{\hat{\bm{\theta}}_{\mathcal{A}}(t)}

\title{Toward Equation of Motion for Deep Neural Networks: Continuous-time Gradient Descent and \\ Discretization Error Analysis}

\author{%
  Taiki Miyagawa
  \\
  NEC Corporation, Japan\\
  \texttt{miyagawataik@nec.com} \\
}

\begin{document}

\maketitle

\begin{abstract}

\input{Abstract}
\end{abstract}

\section{Introduction} \label{sec: Introduction}
\begin{wrapfigure}{r}{0.5\textwidth}
    \centering
    \includegraphics[width=0.95\linewidth]
        {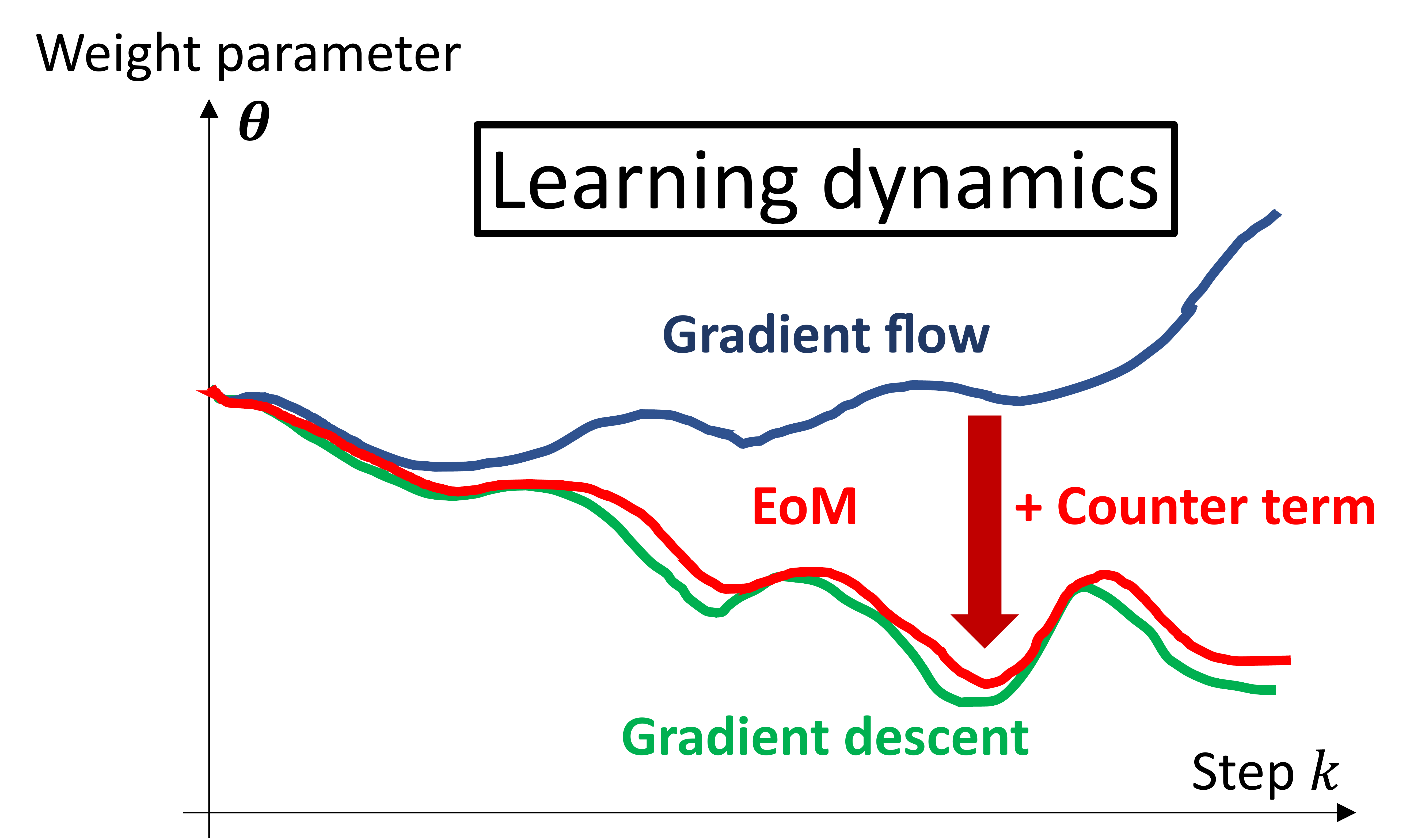}
    \caption{\textbf{Our approach.} GF fails in describing the learning dynamics of GD due to \textit{discretization error}. Our counter term approach successfully cancels the discretization error between GF and GD and hence allows for a reliable analysis of GD.}
    \label{fig: figure1}
\end{wrapfigure}

Let us first explain our primary motivation for the present paper. 
In \textit{physics}, one of the fundamental goals is to predict the dynamics of matter and its fundamental constituents. Specifically, ``predict'' here means to construct differential equations that best describe the physical system under consideration and to solve them. 
Such differential equations are called \textit{Equations of Motion} (EoM). 
An interesting question here may be ``What is the EoM for deep neural networks (DNNs)?''
That is, to what extent can we predict the discrete learning dynamics of DNNs by constructing differential equations? 
This is our research question.

Differential equations have played a prominent role in studying discrete optimization (gradient descent (GD) algorithms), although they are continuous \cite{kushner1978rates, kushner1978stochastic, kushner1984invariant, ljung1992stochastic, kushner2003stochastic, raginsky2012continuous, krichene2015accelerated, mertikopoulos2017convergence, krichene2017acceleration, liu2017stein, feng2017semigroups, scieur2017integration, xu2018accelerated, xu2018continuous_SDE_is_good_for_SGD, alnur2020implicit_ICML2020, kovachki2021continuous, wojtowytsch2021stochastic, elkabetz2021continuous_vs_discrete_NeurIPS2021Spotlight, bu2022feedback, li2022what_ICLR2022Spotlight}.
In the context of deep learning, gradient flow (GF) and stochastic differential equations (SDEs) are used to analyze (stochastic) gradient descent ((S)GD).
Research targets include: convergence \cite{raginsky2012continuous, krichene2015accelerated, mertikopoulos2017convergence, scieur2017integration, xu2018accelerated, krichene2017acceleration, xu2018continuous_SDE_is_good_for_SGD, wojtowytsch2021stochastic}, stability of optimization \cite{bu2022feedback}, optimization with constraints \cite{bu2022feedback}, convergent states \cite{wojtowytsch2021stochastic, li2022what_ICLR2022Spotlight}, flatness of loss landscapes \cite{wojtowytsch2021stochastic}, empirical risk bounds \cite{alnur2020implicit_ICML2020}, and online PCA \cite{feng2017semigroups}. 
Various techniques for continuous analysis have been imported to the analysis of discrete GD algorithms.
However, there still exist gaps between differential equations and actual learning dynamics due to \textit{discretization error}, which is the main interest of the present paper and is often missing in the literature above.
To be specific, we focus on GF $\dbtht = - \bg(\btht)$ as a continuous approximation of GD $\bth_{k+1} = \bth_k - \eta \bg(\bth_k)$, where $\btht \in \mbr^d$ and $\bth_k \in \mbr^d$ are the weight parameters of a DNN at time $t \in \mbr$ and step $k \in \mbz$, respectively, and $\bg \in \mbr^d$ is a gradient vector. $\eta \in \mbr$ is a learning rate and is regarded as the discretization step size when GF is discretized with the Euler method \cite{hairer1993solving}: $\dbth(t=k\eta) \fallingdotseq \f{\bth_{k+1} - \bth_k}{\eta}$.
Due to this approximation, discretization error (or ``continuation error'') is introduced, and thus GF cannot fully explain the dynamics of GD. 
For instance, we show that according to GF, the weight norm of a scale-invariant layer collapses to zero when we use weight decay, while GD does not show such behavior (Section \ref{sec: Learning Dynamics of Scale-invariant Layers}).

To fill the critical gap between GF and GD, we propose modifying GF to describe the learning dynamics of GD more precisely; i.e., we add a counter term $\bx\in\mbr^d$ to the gradient $\bg$ of GF that cancels the discretization error (Figure \ref{fig: figure1}). This idea is motivated by backward error analysis in numerical analysis \cite{hairer1993solving}. 
We derive a functional integral equation that determines the counter term and solve it (Section \ref{sec: Equation of Motion of Deep Neural Networks}). 
As a result, we obtain a more reliable differential equation, called \textit{EoM} here, that describes the discrete learning dynamics of GD.
Using the counter term, we derive the leading order of discretization error (Section \ref{sec: Leading Order of Discretization Error Is Given By Counter Term}) to show to what extent GF and EoM are precise in describing GD's dynamics.
This point is often missed in the literature on the continuous approximation of discrete GD algorithms \cite{li2017stochastic_modified_eq._ICML2017, li2019stochastic_modified_equations_math_found_JMLR2019, feng2020uniform, feng2017semigroups, hu2019on_the_diffusion_AMSA2019, an2019stochastic_modified_eq._asynchro., barrett2021implicit_gradient_regularization_ICLR2021, smith2021on_the_origin_of_IR_in_SGD}.
We further derive a sufficient condition for learning rates for the discretization error to be small (Section \ref{sec: Discretization Error Bounds}).
We show that EoM well explains empirical results.

Furthermore, to show the benefits of EoM, we apply it to two specific cases: scale-invariant layers \cite{van2017l2_effectiveLR_orginal1, zhang2018three} and translation-invariant layers \cite{kunin2021neural_mechanics_1, tanaka2021noethers_neural_mechanics_2} (Section \ref{sec: Applications}). 
For scale-invariant layers, we show that a better description of GD's discrete dynamics requires modifications to the decay rate of weight norms that is previously derived in the continuous regime (SDEs) \cite{zhiyuan2020reconciling_NIPS2020}.
In addition, we show that EoM successfully reproduces the limiting dynamics ($t\rightarrow\infty$) of weight norms and angular update \cite{wan2021spherical_motion_dynamics} that are previously derived in the discrete regime, while GF cannot reproduce this result.
For translation-invariant layers, we show that EoM rather than GF dramatically matches empirical results, indicating the importance of the counter term.
To the best of our knowledge, no study analyzes the temporal evolution of translation-invariant layers except for \cite{kunin2021neural_mechanics_1} and \cite{tanaka2021noethers_neural_mechanics_2}, where only the sum of weights is their focus, while we derive the dynamics of the whole weights.


Our contribution is four-fold. Our code\footnote{See Supplementary Materials at \url{https://openreview.net/forum?id=qq84D17BPu} .} and detailed experimental results are given as supplementary materials.
\begin{enumerate}
\item To fill the critical gap between GF and GD, we derive a counter term for GF that cancels the discretization error, and as a result, we obtain EoM, a continuous differential equation that precisely describes the discrete learning dynamics of GD.
\item To show to what extent GF and EoM are precise in describing discrete GD dynamics, we derive the leading order of discretization error, as is often missed in the literature on the continuous approximation of discrete GD algorithms.
We further derive a sufficient condition for learning rates for the discretization error to be small.
\item We apply EoM to two specific cases: scale-invariant layers and translation-invariant layers, indicating the importance of the counter term for a better description of the discrete learning dynamics of GD.
\item Our experimental results support our theoretical findings.
\end{enumerate}
Our work is the first step toward answering this research question: to what extent can we predict the discrete learning dynamics of DNNs by constructing differential equations (EoM for DNNs)?
Also, our work helps researchers import continuous analysis to the discrete analysis of GD algorithms. In this sense, our work bridges discrete and continuous analyses of GD algorithms.

\section{Related Work} \label{sec: Related Work}
The idea of approximating discrete-time stochastic algorithms with continuous equations dates back to stochastic approximation theory \cite{kushner1978rates, kushner1978stochastic, kushner1984invariant, ljung1992stochastic, kushner2003stochastic}.
Their primary focus is convergence analysis for discrete-time algorithms, while our focus is to predict the learning dynamics (temporal evolution) of weight parameters, such as the decay rates of weight norms and effective learning rate of scale-invariant layers.
Our idea of the counter term is inspired by the backward error analysis developed for numerical analysis \cite{hairer2006geometric_numerical_integration_book}.
This idea is now used to analyze discrete optimization \cite{li2017stochastic_modified_eq._ICML2017, li2019stochastic_modified_equations_math_found_JMLR2019, feng2020uniform, feng2017semigroups, hu2019on_the_diffusion_AMSA2019, an2019stochastic_modified_eq._asynchro., barrett2021implicit_gradient_regularization_ICLR2021, smith2021on_the_origin_of_IR_in_SGD}. 
\cite{elkabetz2021continuous_vs_discrete_NeurIPS2021Spotlight} is a pioneering work on discretization error analysis between GF and GD that is based on the numerical analysis of the Euler method \cite{hairer1993solving}. They derive a sufficient condition for learning rates for the discretization error to be small. 
This analysis is based on a bound (inequality), while we derive an explicit relationship between learning rates and discretization error as an equality.

Neural mechanics and Noether's learning dynamics \cite{kunin2021neural_mechanics_1, tanaka2021noethers_neural_mechanics_2} provide a solution to a part of the aforementioned problem: to what extent can we predict the learning dynamics of DNNs by constructing differential equations? They derive (the breaking of) conservation laws of weight parameters using differential equations and provide the temporal evolution of the conserved quantities. The present work is inspired by these studies but has crucial differences: 1) our focus is on the temporal evolution of all of the network parameters, not only the conserved quantities, 2) the gradient's correction for canceling the discretization error is not limited to the first order, but all orders, and 3) the discretization error is explicitly provided in the present paper. See Appendix \ref{app: Supplementary Discussion} for more related studies.

\section{Equation of Motion for Deep Neural Networks} \label{sec: Equation of Motion of Deep Neural Networks}
In the following sections, we define \textit{EoM} by modifying GF (Section \ref{sec: Our Approach and Definitions}). We show that the counter term satisfies a functional integral equation (Section \ref{sec: How to Determine Counter Term}), and then we solve it (Section \ref{sec: Solution to EoLDE}).

\subsection{Our Approach and Definitions} \label{sec: Our Approach and Definitions}
We begin with a simple idea: add a counter term to GF to cancel discretization error, i.e.,
\begin{align} 
    \dbtht = - \bm{g} (\btht) - \eta \bm{\xi} (\btht) \, , 
    \label{eq: gradient flow}
\end{align}
where $\btht \in \mbr^d$ is the vectorized weight parameters of a DNN at time $t \in \mbr$, $d \in \mbn$ is the dimension of the weight, and $\dbtht$ denotes $d\btht / dt$.
Gradient $\bg(\btht)$ is defined as $\bm{g} (\btht) := \nabla f (\btht) + \la \btht$, which consists of a loss function $f(\btht)$ and weight decay term $\la \btht$, where $\la >0$ controls the strength of weight decay. 
$\eta > 0$ is a small learning rate, and $\bm{\xi} (\btht) \in \mbr^d$ is the counter term.
Throughout this paper, we assume all functions are sufficiently smooth.
We call Equation (\ref{eq: gradient flow}) the \textit{Equation of Motion (EoM)} for DNNs, or simply EoM.

Our aim is to find $\bx$ that makes Equation (\ref{eq: gradient flow}) more reliable to precisely approximate GD $\bth_{k+1} = \bth_k - \eta \bm{g}(\bth_k)$,
where $\bth_k \in \mbr^d$ is the weight at step $k\in\mbz_{\geq 0}$. 
To do so, we first define the \textit{discretization error} between GF (\ref{eq: gradient flow}) and GD at step $k$:
\begin{align}
    \berr_k := \bth(k\eta) - \bth_k \,\, \in \mbr^d \, 
    \label{eq: definiton of discretization error}
\end{align}
and find $\bx$ that makes $\berr_k$ small. Throughout this paper, we use the standard Euler method to discretize GF: $\dbtht \fallingdotseq (\bth(t+\eta) - \btht) / \eta$ and  $t = k\eta$; thus, $\eta$ is identified with the discretization step size. 

\subsection{How to Determine Counter Term} \label{sec: How to Determine Counter Term}
We show that the leading order of $\berr_k$ with respect to $\eta$ is controlled by the counter term (Theorem \ref{thm: Leading order of EoDE}), and as a result, the counter term is determined via a functional integral equation (Equation (\ref{eq: EoLDE})).

Our first theorem shows what the counter term should cancel.
\begin{theorem}[Recursive formula for discretization error] \label{thm: EoDE}
Discretization error $\berr_k$ 
satisfies:
\begin{align}
    \berr_{k+1} - \berr_k 
    &= -\eta \big( \bm{g}(\bth(k\eta)) - \bm{g} (\bth(k\eta) - \berr_k) \big) 
        +\eta^2\int_0^1 ds \ddot{\bth} (\eta(k+s)) (1-s) 
        -\eta^2\bm{\xi}(\bth(k\eta))
        \label{eq: EoDE} \\
    &=: -\eta \left( \bm{g}(\bth(k\eta)) - \bm{g} (\bth(k\eta) - \berr_k) \right) + \bm{\La}(\bthketa) \, .
\end{align}
\end{theorem}
Here, we defined $\bm{\La} (\bth (k\eta)) := \eta^2 \int_0^1 ds \ddot{\bth}(\eta(k+s)) (1-s) -\eta^2\bm{\xi}(\bth(k\eta)) \,\,\, \in \mbr^d$.
The proof is based on Taylor's theorem and is given in Appendix \ref{app: Proof of thm: EoDE}.
The right-hand side of Equation (\ref{eq: EoDE}) tells us that the counter term (third term) should cancel the first and second terms.
However, the following theorem states that the first term gives only subleading contributions with respect to $\eta$.
\begin{theorem}[Leading order of discretization error] \label{thm: Leading order of EoDE}
Suppose that $ \bm{\La} (\bth (k\eta)) = O (\eta^\ga)$ and $\berr_0 = O (\eta^\ga)$ for some $\ga > 0$. Then $\berr_k = O (\eta^\ga)$ and
$- \eta ( \bg(\bth(k\eta)) - \bg(\bth(k\eta) - \berr_k) ) = O(\eta^{\ga+1})$. Therefore, the first term in the right-hand side of Equation (\ref{eq: EoDE}) is negligible compared with $\bm{\La}$: 
\begin{align}
    \berr_{k+1} 
    &= \berr_k + \bm{\La} (\bth (k\eta)) - \eta ( \bg(\bth(k\eta)) - \bg(\bth(k\eta) - \berr_k) ) \nn 
    &= \berr_k + \bm{\La} (\bth (k\eta)) + O(\eta^{\ga+1}) \hspace{20pt} (k = 0,1,2,...) \, . 
    \label{eq: disc. err. in thm: Leading order of EoDE}
\end{align}
\end{theorem}
The proof is by induction and given in Appendix \ref{app: Proof of thm: Leading order of EoDE}.
Therefore, the leading order of discretization error is $O(\eta^\ga)$ and given by:
\begin{align}
    \bm{\La} (\bth(k\eta)) = O(\eta^\ga) \,
    \Longleftrightarrow \,
    \int_0^1 ds \, \ddot{\bth} (\eta(k+s)) (1-s) - \bm{\xi} (\bth (k\eta))  = O(\eta^{\ga-2})\, .
    \label{eq: EoLDE}
\end{align}
This is a functional equation of $\bx$ because $\ddbtht$ contains $\bx$ via Equation (\ref{eq: gradient flow}).
A solution to Equation (\ref{eq: EoLDE}) for a large $\ga$ gives a small $\bm{\La}$ and thus gives a small $\berr_k$ via Equation (\ref{eq: disc. err. in thm: Leading order of EoDE}). 

\subsection{Solution to Equation \ref{eq: EoLDE}} \label{sec: Solution to EoLDE}
How can we solve Equation (\ref{eq: EoLDE})? It is not easy to find an exact solution because Equation (\ref{eq: EoLDE}) is a functional integral equation \cite{rus1981problem, lungu2009functional, long2013nonlinear, thuyet2014nonlinear, marian2021functional}; therefore, we assume a power series solution with respect to $\eta$: 
\begin{align}
    \bm{\xi} (\bth (k\eta)) 
    = \sum_{\al=0}^\infty \eta^\al \bx_\al 
    = \bm{\xi}_0 (\bth (k\eta)) + \eta \bm{\xi}_1 (\bth (k\eta)) + \eta^2 \bm{\xi}_2 (\bth (k\eta)) + \cdots \, .
    \label{eq: counter term expanded}
\end{align}
In the following theorem, we successfully find a solution for \textit{all} orders of $\eta$. 
\begin{theorem}[Solution of Equation \ref{eq: EoLDE}] \label{thm: Solution of EoLDE}
The solution to Equation (\ref{eq: EoLDE}) of form (\ref{eq: counter term expanded}) is given by
\begin{align}
    \bx_\al(\bth) = \ti{\bx}_\al(\bth) 
    := \sum_{i=2}^{\al+2} \sum_{k_1+\cdots+k_i = \al - i + 2} \f{(-1)^i}{i!} D_{k_1} \cdots D_{k_{i-1}} \Xi_{k_i}
    \label{eq: Solution of EoLDE}
\end{align}
for $\al=0, 1, 2,...$, where we use differential operators (Lie derivatives)
$\mcd_\al := \ti{\bx}_{\al-1}(\bth) \cdot \nabla \,\, (\al = 1, 2,...)$ and $\mcd_0 := \bg(\bth) \cdot \nabla$
and also defined $\Xi_\al(\bth) := \ti{\bx}_{\al-1}(\bth)$ ($\al = 1,2,...\,$) and $\,\Xi_{0}(\bth) := \bg(\bth)$. 
\end{theorem}
The proof follows from the definition of the Lie derivative and is given in Appendix \ref{app: Proof of thm: Solution of EoLDE}. The first two orders of the solution are given by:
\begin{align}
    &\tbx_0(\bth) 
        = \f{1}{2} \, (\bg (\bth) \cdot \nabla) \bg (\bth)
        = \f{1}{4}\nabla || \bg (\bth)  ||^2 
        \label{eq: hat xi0} \\
    &\tbx_1(\bth) 
        = \f{1}{2} (\tbx_0(\bth) \cdot \nabla) \bg (\bth) 
        + \f{1}{6} (\bg (\bth) \cdot \nabla) \tbx_0
        \label{eq: hat xi1} 
        \, .
\end{align}

\paragraph{Discussions.}
As can be inferred from Equations (\ref{eq: Solution of EoLDE}--\ref{eq: hat xi1}), $\tbx_\al$ contains the $\al+2_{\rm nd}$-order derivative of the loss function. Therefore, the higher-order counter terms cancel the higher-order smoothness of the discretization error.

Here, we note that Equation (\ref{eq: Solution of EoLDE}) can be found, e.g., in \cite{hairer2006geometric_numerical_integration_book}, as a higher-order backward error analysis. However, our derivation above has independent contributions: 1) we clarify that the counter term cancels the leading order of discretization error (Theorem \ref{thm: Leading order of EoDE}), and 2) we find that the discretization error itself is also given by the counter term (Corollary \ref{cor: Discretization error of bx} in the next section).

Equation (\ref{eq: hat xi0}) often appears in the literature on backward error analysis \cite{hairer1993solving, hairer2006geometric_numerical_integration_book} and its related topics in machine learning, e.g., \cite{li2017stochastic_modified_eq_ICML2017, li2019stochastic_modified_equations_math_found_JMLR2019, feng2020uniform, barrett2021implicit_gradient_regularization_ICLR2021, smith2021on_the_origin_of_IR_in_SGD, kunin2021neural_mechanics_1}. 
Typically, $\tbx_0$ is added to gradients of continuous equations (e.g., SDE) to close the gap between continuous equations and discrete algorithms (e.g., SGD) by canceling (at least first-order) discretization error. 
However, higher-order discretization error is neglected in these studies.
In contrast, our solution (\ref{eq: Solution of EoLDE}) cancels \textit{all} orders of discretization error.

\section{Discretization Error} \label{sec: Discretization Error}
The question here is to what extent the continuous approximation (\ref{eq: gradient flow}, \ref{eq: Solution of EoLDE}) is precise; this point is often missed in the literature on continuous approximation \cite{li2017stochastic_modified_eq._ICML2017, li2019stochastic_modified_equations_math_found_JMLR2019, feng2020uniform, feng2017semigroups, hu2019on_the_diffusion_AMSA2019, an2019stochastic_modified_eq._asynchro., barrett2021implicit_gradient_regularization_ICLR2021, smith2021on_the_origin_of_IR_in_SGD}.
In this section, we use the counter term (\ref{eq: Solution of EoLDE}) and quantify discretization error as a function of the loss function and its derivatives (Section \ref{sec: Leading Order of Discretization Error Is Given By Counter Term}).
We find that our result well explains empirical results.
We further derive a sufficient condition for learning rates for the discretization error to be small (Section \ref{sec: Discretization Error Bounds}).

\subsection{Counter Term Gives Leading Order of Discretization Error} \label{sec: Leading Order of Discretization Error Is Given By Counter Term}
We show that the counter term gives the leading order of discretization error between GD vs. GF and EoM. The proof follows from Theorem \ref{thm: Leading order of EoDE} and \ref{thm: Solution of EoLDE} and is given in Appendix \ref{app: Proof of cor: Discretization error of bx}.
\begin{corollary}[Leading order of discretization error is given by $\tbx_\al$] \label{cor: Discretization error of bx}
Suppose that we use $\bx$ up to $O(\eta^{\ga-1})$, i.e., $\bx = \tbx_0 + \eta \tbx_1 + \dotsm + \eta^{\ga-1} \tbx_{\ga-1}$ for $\ga \in \mbz_{>0}$ ($\bx := \bm{0}$ for $\ga=0$). Then, 
\begin{align}
    \berr_{k+1}  
    = \berr_k + \bm{\La}(\bthketa) + O(\eta^{\ga+3})
    = \berr_k + \eta^{\ga+2} \tbx_\ga + O(\eta^{\ga+3}) \, .
    \label{eq: discretization error}
\end{align}
\end{corollary}
First, Corollary \ref{cor: Discretization error of bx} implies that the higher the orders of the counter term we use (large $\gamma$), the more precise EoM (\ref{eq: gradient flow}) is (small $\berr_k$). Thus, GF ($\bx=\bm{0}$) gives larger discretization error than EoM ($\bx\neq \bm{0}$).
Second, Corollary \ref{cor: Discretization error of bx} gives the \textit{equality} of the leading order of discretization error at \textit{arbitrary} steps. This is not a \textit{bound} \cite{elkabetz2021continuous_vs_discrete_NeurIPS2021Spotlight} nor an \textit{asymptotic} analysis ($k\rightarrow\infty$).
Third, let us give an intuition by considering $\bx = \bm{0}$ (GF). Then, Corollary \ref{cor: Discretization error of bx} gives: 
\begin{align}
    \berr_{k+1} 
    &= \berr_0 + \sum_{s=0}^k \f{\eta^2}{2}(H(\bth (s\eta)) + \la I)(\nabla f(\bth(s\eta)) + \la \bth(s\eta)) + O(\eta^3) \, , \label{eq: disc. err. with no xi} 
\end{align}
where $H(\bth) \in \mbr^{d \times d}$ is the Hessian of the loss function $f$ with respect to $\bth$ and $I \in \mbr^{d \times d}$ is the identity matrix. Equation (\ref{eq: disc. err. with no xi}) suggests that 1) large learning rates lead to a large discretization error and 2) steep loss functions (along the trajectory) lead to a large discretization error. 

\begin{wrapfigure}{r}{0.5\textwidth}
  \centering
	\includegraphics[width=0.95\linewidth]
        {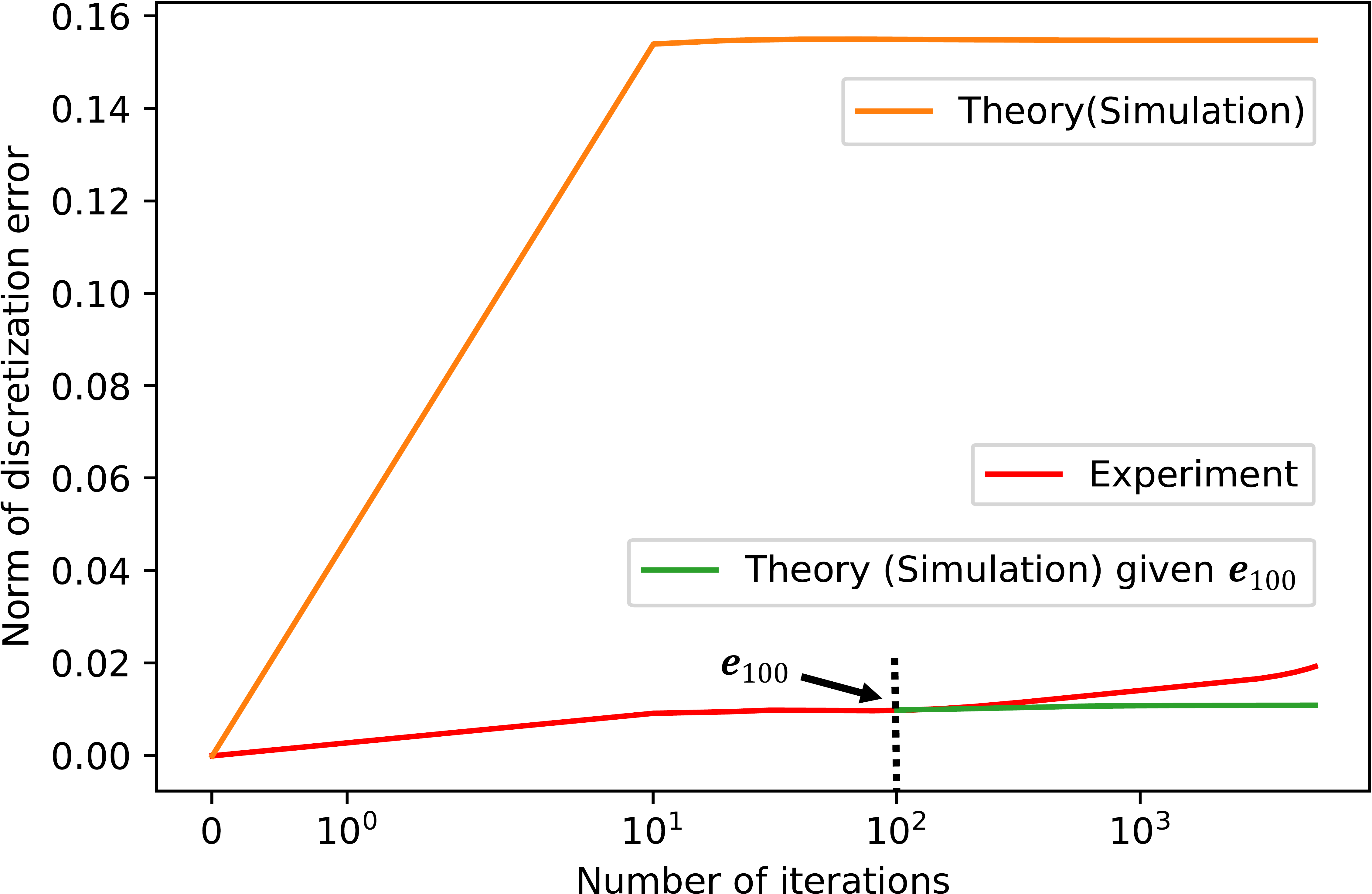}
	\caption{\textbf{Theoretical prediction of discretization error of GF and GD (Equation (\ref{eq: disc. err. with no xi})) vs. actual discretization error of GF and GD.} 
	The larning rate and weight decay are $10^{-2}$ and $10^{-2}$.
	See Appendix \ref{app: Theoretical prediction vs. experimental result of discretization error} for more results and details.
	See Section \ref{sec: Experiment} for experimental settings.}
	\label{fig: th vs ex discerr}
\end{wrapfigure}


\paragraph{Empirical result.}
We find Equation (\ref{eq: disc. err. with no xi}) well explains our empirical result.
We compare Equation (\ref{eq: disc. err. with no xi}) (up to $O(\eta^2)$) with the actual discretization error of GD and GF in Figure \ref{fig: th vs ex discerr}.
First, the gap between our theoretical prediction of discretization error (orange curve) and the actual discretization error (red curve) is small because the range of \textit{relative} error ($||\berr_k||/||\bth_k||$) in this plot is only 0--0.01 (see also Figure \ref{fig: relaive disc. err.} in Appendix \ref{app: Supplementary Experiment}).
Second, most of the discretization error for Theory (orange curve) and Experiment (red curve) is produced within the first 100 steps.
We can understand this phenomenon with the help of Equation (\ref{eq: disc. err. with no xi}). It suggests that discretization error can be enhanced when the loss function is non-smooth along the learning trajectory, which is likely to occur at the beginning of training due to random initialization.
Therefore, a large part of discretization error is produced in the early stage of training.
Third, we see that most of the gap between Theory (orange curve) and Experiment (red curve) also comes from the first 100 steps; in fact, the green curve shows that there is a much smaller enhancement of the gap after the 100th step. 
The source of the gap is the higher-order term $O(\eta^3)$ in Equation (\ref{eq: disc. err. with no xi}). 
It consists of higher-order derivatives of the loss function (Theorem \ref{eq: Solution of EoLDE} and Corollary \ref{cor: Discretization error of bx}) and thus can be large when the loss function is non-smooth along the learning trajectory. 
Therefore, by the same logic as above, the early stage of training tends to produce a gap between Theory (orange curve) and Experiment (red curve). 


\subsection{Discretization Error Bounds} \label{sec: Discretization Error Bounds}
We provide a sufficient condition (an upper bound for $\eta$) for GF and EoM to follow GD up to a given step $k$, which helps us infer desired learning rates (step sizes) for the discretization error to be small.
We first consider $\bx=\bm{0}$ (GF).  
\begin{corollary}[Learning rate bound for $\bx=\bm{0}$] \label{cor: Learning rate bound when bx=0}
Let $\bx=\bm{0}$ and assume that $\berr_0=O(\eta^3)$. Let $\ep$ and $t$ be arbitrary positive numbers. If the step size satisfies
\begin{align}
    \eta < \sqrt{\f{\ep}{k}} \sqrt{
        \f{2}{
             \underset{0\leq t^\prime \leq t}{\max} \{ || ( H(\bth(t^\prime)) + \la I ) \bg (\bth(t^\prime)) || \}
        }
    } \, ,
    \label{eq: LR bound1}
\end{align}
for some $k \in \{1,2,..., \lfloor\f{t}{\eta}\rfloor \}$, then the discretization error can be arbitrarily small:
\begin{align}
    ||\berr_k|| < \ep + O(\ep^{\f{3}{2}}) \, .
\end{align}
\end{corollary}
The proof follows from Equation (\ref{eq: disc. err. with no xi}) and is given in Appendix \ref{app: Proof of cor: Learning rate bound when bx=0}.
We see that
1) there is no guarantee that the discretization error is small unless the learning rate is sufficiently small,
2) we need small learning rates to keep the discretization error small for a long period, and 
3) we need small learning rates to keep the discretization error small for non-smooth loss landscapes.
This is consistent with our empirical results in Figure \ref{fig: discretization error} and \ref{fig: discretization error magnf}; in fact,
1) the discretization error blows up for a large learning rate ($\eta=10^{-1}$ in Figure \ref{fig: discretization error}),
2) it increases as the number of steps increases (Figure \ref{fig: discretization error magnf}), and 
3) most of it is produced in the early phase of training, where the objective function tends to be non-smooth, and the gradients tend to be large.

We compare our bound (\ref{eq: LR bound1}) with a bound given in \cite{elkabetz2021continuous_vs_discrete_NeurIPS2021Spotlight} because, to our knowledge, only \cite{elkabetz2021continuous_vs_discrete_NeurIPS2021Spotlight} provides a bound for the step size with respect to discretization error in the context of deep learning.
In \cite{elkabetz2021continuous_vs_discrete_NeurIPS2021Spotlight}, it is proved that in essence, $\eta \lesssim \ep/\be_{t \ep}\ga_{t \ep} c_t$, where $\be_{t \ep}$ and $\ga_{t \ep}$ measure the non-smoothness of the loss function, and $c_t$ depends on the spectrum of the Hessian. These factors are hard to compute analytically unless the loss function and network are simple, but the qualitative behavior of this bound is the same as ours (\ref{eq: LR bound1}); i.e., both bounds become tight when the loss function is non-smooth. 


We also derive a learning rate bound for $\bx = \tbx_0$ (EoM) and the full statement is given in Corollary \ref{cor: Learning rate bound when bx=tbx0} in Appendix \ref{app: Proof of cor: Learning rate bound when bx=tbx0}, which states that if $\eta < O (\sqrt[3]{\f{\ep}{k}})$, then $||\berr_k|| < \ep + O(\ep^{\f{4}{3}})$. Therefore, larger step sizes are now allowed compared with Corollary \ref{cor: Learning rate bound when bx=0} (GF) because of the non-zero counter term. Furthermore, we can show larger bounds for higher-order counter terms in a similar way. 

\begin{figure}[htbp]
    \begin{minipage}[t]{0.5\linewidth}
        \centering
    	\includegraphics[width=0.9\linewidth]
          {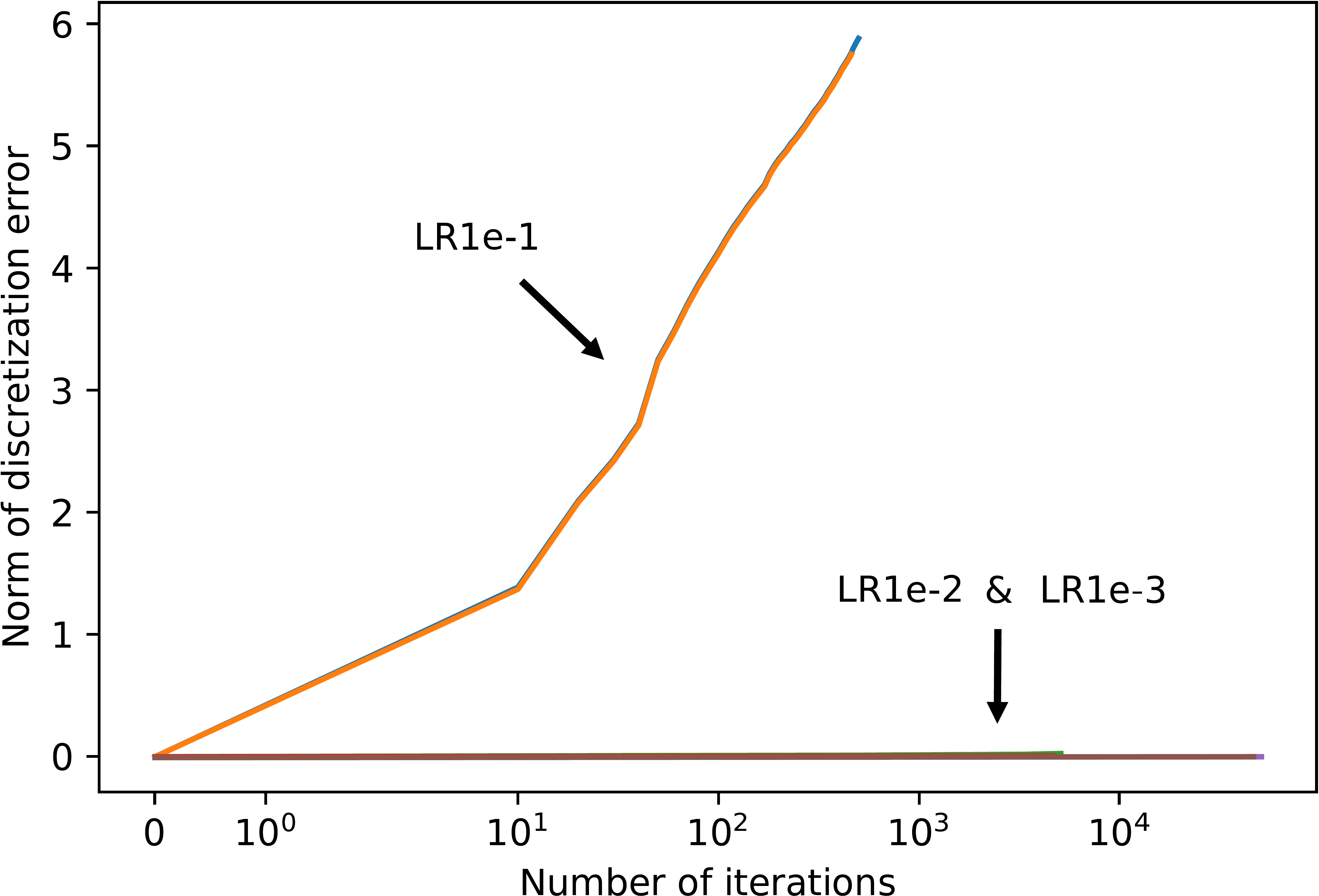}
    	\caption{\textbf{Discretization error explodes for large learning rate ($10^{-1}$).} LR means learning rate. Weight decay is $10^{-3}$. Curves include both GF and EoM. Relative discretization error is also shown in Appendix \ref{app: Supplementary Experiment}. See Section \ref{sec: Experiment} for experimental settings.}
    	\label{fig: discretization error}
    \end{minipage} \,
    \begin{minipage}[t]{0.5\linewidth}
        \centering
    	\includegraphics[width=0.9\linewidth]
          {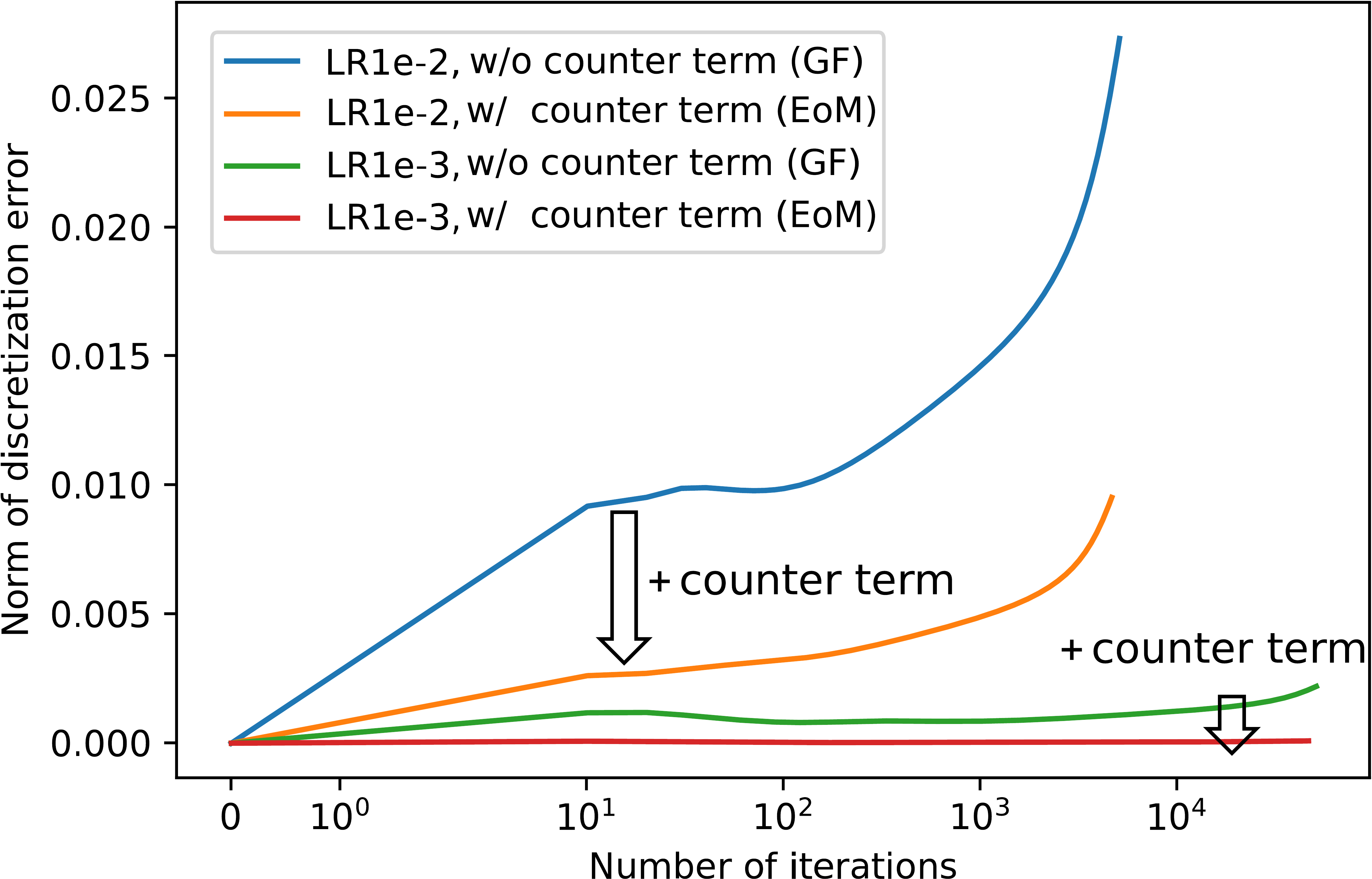}
    	\caption{\textbf{Discretization error of GF and EoM.} Figure \ref{fig: discretization error} is magnified. The counter term reduces discretization error as expected, and smaller learning rates give smaller discretization errors.}
    	\label{fig: discretization error magnf}
    \end{minipage}
\end{figure}

\section{Application: Scale- and Translation-invariant Layers} \label{sec: Applications}
To show the benefits of EoM, we finally apply our theory to two specific cases: scale-invariant layers \cite{van2017l2_effectiveLR_orginal1, zhang2018three} and translation-invariant layers \cite{kunin2021neural_mechanics_1, tanaka2021noethers_neural_mechanics_2}. 
Additionally, Appendix \ref{app: Broken Conservation Laws} provides an application to broken conservation laws \cite{kunin2021neural_mechanics_1}.
In the following, we simply focus on $\bx = \bm{0}$ and $\bx=\tbx_0$ to analyze the differences between $\bx = \bm{0}$ and $\bx \neq \bm{0}$.

\paragraph{Definitions}
Let us first introduce our notation. 
A transformation $\bpsi$ of $\bth \in \mbr^d$ with parameter $\al \in \mbr$ is said to be a \textit{symmetry transformation} of loss function $f$ if $f(\bpsi (\bth, \al)) = f(\bth)$. 
$\mbone_\mca \in \{0,1\}^d$ denotes the indicator vector of subspace $\mca \subset \mbr^d$ (e.g., $\mca$ is a linear layer in the DNN). For a scalar $\al \in \mbr$, we define $\al_\mca := \al \mbone_\mca + \mbone_{\mca^\mathsf{c}} \in \mbr$, where $\mca^\mathsf{c}$ is the complement of $\mca$. 
For a vector $\bth \in \mbr^d$, we define $\bth_\mca := \bth \odot \mbone_\mca \in \mbr^d$, where $\odot$ is the Hadamard element-wise product. 
For the gradient operator $\nabla = (\partial/\partial\th_1, ..., \partial/\partial\th_d)^\top$, we define $\nabla_\mca := \mbone_\mca \odot\nabla$.
We also define $r_\mca := || \btha ||$ and $\hbth_\mca := \btha / r_\mca$. 


\subsection{Learning Dynamics of Scale-invariant Layers} \label{sec: Learning Dynamics of Scale-invariant Layers}
In this section, we focus on scale-invariant layers. 
A scale-invariant layer $\mca$ is defined as a subspace that is invariant under the scale transformation $\bpsi (\bth, \al) := \al_\mca \bth = \al \btha + \bth_{\mca^\mathsf{c}}$ ($\al > 0$). 
For example, a linear layer immediately before a batch normalization layer is scale-invariant.
We see that for a better description of GD's discrete dynamics, we need modifications to the decay rate of $r_\mca$ that is previously derived in the continuous regime \cite{zhiyuan2020reconciling_NIPS2020}.
In addition, we show that EoM successfully reproduces the limiting dynamics of $r_\mca$ and \textit{angular update} \cite{wan2021spherical_motion_dynamics} at $t\rightarrow\infty$ that are previously derived in the discrete regime, while GF cannot.
In Appendix \ref{app: Equation of Motion for hbth}, we additionally show that there are crucial differences between GD and GF via the \textit{effective learning rate} of scale-invariant layers \cite{van2017l2_effectiveLR_orginal1, hoffer2018norm, zhang2018three, arora2018theoretical, chiley2019online, zhiyuan2020reconciling_NIPS2020, li2020exponential_learning_rate_ICLR2020, wan2021spherical_motion_dynamics, li2022robust, roburin2022spherical}.


\paragraph{EoM for $r$} 
We construct the EoM for $r_\mca$ (the EoM for $\hbtha$ is given in Appendix \ref{app: Equation of Motion for hbth} for completeness).
\begin{theorem}[EoM for $r_\mca$ and solution] \label{thm: EoM of r}
EoM (\ref{eq: gradient flow}) gives $\dot{r_\mca^2}(t) = -2 \la r_\mca^2(t) -2 \, \eta \, \bth_\mca(t) \cdot \bx(\btht)$.
Specifically, this is equivalent to:
\begin{align}
    \dot{r_\mca^2}(t) = -2 \la r_\mca^2(t) 
    \Longleftrightarrow \,\, r_\mca^2(t) = r_\mca^2(0) e^{-2\la t} 
    \label{eq: Solution of EoM for r w/ xi=0}
\end{align}
for $\bx = \bm{0}$ (GF) and
\begin{align}
    & \,\, \dot{r_\mca^2}(t) 
        = -2 (\la + \f{\eta\la^2}{2}) r_\mca^2(t) + \f{\eta}{r_\mca^2(t)} || \nabla_\mca f(\hbth_\mca(t) + \bthac (t)) ||^2  \\
    \Longleftrightarrow & \,\, r_\mca^2(t) = r_\mca^2(0) e^{-2\la (1 + \f{\eta \la}{2}) t} + \eta \int_0^t d\tau e^{- 2 \la (1 + \f{\eta \la}{2})(t-\tau)} \f{|| \nabla_\mca f (\hbth_\mca(\tau) + \bthac(\tau)) ||^2}{r_\mca^2(\tau)} 
    \label{eq: Solution of EoM for r w/ xi=xi0}
\end{align}
for $\bx = \tbx_0$ (EoM).
\end{theorem}
The proof is based on Equations (\ref{eq: gradient flow}, \ref{eq: hat xi0}) and given in Appendix \ref{app: Proof of thm: EoM of r}.
Equation (\ref{eq: Solution of EoM for r w/ xi=0}) gives $r_\mca^2(k\eta) = r_\mca^2(0)e^{-2\eta\la k}$ ($k \in \mbz_{\geq0}$) at discretization; therefore, $\eta\la$ is regarded as the decay rate of $r_\mca$ (\textit{intrinsic learning rate} \cite{zhiyuan2020reconciling_NIPS2020}).
This is originally discussed in the continuous regime (SDE) \cite{zhiyuan2020reconciling_NIPS2020}; however, we find that for a better description of the discrete dynamics of GD, the decay rate needs to be modified from $\eta\la$ to $\eta\la ( 1 + \f{\eta\la}{2})$  (see the exponent of Equation (\ref{eq: Solution of EoM for r w/ xi=xi0})).
This means that $r_\mca$ in GD decays faster than expected from a naive continuous dynamics (GF (\ref{eq: Solution of EoM for r w/ xi=0}) and SDE \cite{zhiyuan2020reconciling_NIPS2020}). See Appendix \ref{app: Supplementary Discussion} for higher-order corrections.



\paragraph{Limiting dynamics.}
We next derive the limiting dynamics ($t\rightarrow \infty$) of $r_\mca$.
\begin{corollary}[$r_\mca$ at equilibrium] \label{cor: r at equil.}
When $\bx=\bm{0}$ (GF), $r_\mca$ collapses to zero as $t \rightarrow \infty$. 
When $\bx = \tbx_0$ (EoM), assume that there exist two constants $r_{\mca *} \geq 0$ and $c_* \geq 0$ such that $r_\mca(t) \xrightarrow{t\rightarrow\infty} r_{\mca *}$ and $|| \nabla_\mca f(\hbth_\mca(t) + \bthac (t)) || \xrightarrow{t\rightarrow\infty} c_*$. Then $r_{\mca *}^2 = \sqrt{\f{\eta}{2\la+\eta\la^2}} c_*$.
\end{corollary}
The proof follows from Theorem \ref{thm: EoM of r} and is given in Appendix \ref{app: Proof of cor: r at equil.}.
The non-zero counter term successfully reproduces $r_{\mca *}^2 \sim \sqrt{\eta/2\la}\,c_*$ \cite{van2017l2_effectiveLR_orginal1, wan2021spherical_motion_dynamics}, which is originally derived in the discrete regime (SGD), although our approach is continuous (EoM (\ref{eq: gradient flow})).
Without the counter term, we cannot explain this behavior because GF gives $r_\mca(t) \xrightarrow{t\rightarrow\infty} 0 (\neq \sqrt{\eta/2\la}\,c_*)$. 


We next derive the limiting dynamics of \textit{angular update} \cite{wan2021spherical_motion_dynamics}, which is designed to measure the temporal evolution of scale-invariant networks.
It is originally defined in the discrete regime: $\cos \De_k := \hat{\bth}_{\mca k} \cdot \hat{\bth}_{\mca k+1}$, where $\hbth_{\mca k} := \f{\mbone_\mca \odot \bth_{k}}{|| \mbone_\mca \odot \bth_{k} ||}$.
That is, $\De_k$ represents a single-step angular change in the weight parameters of the scale-invariant layers $\mca$. 
In the continuous regime, we can define $\cos \De(t) := \hbtha(t) \cdot \hbtha(t+\eta)$.
\begin{corollary}[$\De(t)$ at equilibrium] \label{cor: angular update at equil.}
Let us use $\bx = \tbx_0$. Suppose that the assumptions in Corollary \ref{cor: r at equil.} are satisfied. The angular update at equilibrium, denoted by $\De_*$, is given by $\cos \De_* = \f{1-\eta\la}{1 - \eta^2\la^2 /2} + O(\eta^3)$,
and thus, $\De_* = \sqrt{2\eta\la} + O((\eta\la)^{3/2})$.
\end{corollary}
The proof is based on Corollary \ref{cor: r at equil.} and is given in Appendix \ref{app: Proof of cor: angular update at equil.}.
EoM successfully reproduces $\De_* \sim \sqrt{2\eta\la}$ \cite{wan2021spherical_motion_dynamics}, which is originally derived in the discrete regime (SGD), although EoM is continuous itself.
On the other hand, GF cannot explain the limiting dynamics of $\De(t)$ because when $\bx = \bm{0}$, $r(t)$ goes to zero as $t\rightarrow\infty$ (Equation (\ref{eq: Solution of EoM for r w/ xi=0})), and thus, $\cos \De(t) = \f{\bthat}{r_\mca (t)}\cdot\f{\btha(t+\eta)}{r_\mca(t+\eta)}$ is ill-defined.
In summary, there are gaps between GF and GD, and our discussion above indicates that the counter term is inevitable to describe the actual dynamics of GD.

\subsection{Learning Dynamics of Translation-invariant Layers}
Next, we apply EoM to translation-invariant layers.
To the best of our knowledge, no study analyzes the temporal evolution of translation-invariant layers except for \cite{kunin2021neural_mechanics_1} and \cite{tanaka2021noethers_neural_mechanics_2}, 
where only the sum of weights is their focus, while we derive the dynamics of the whole weights.
A translation-invariant layer $\mca$ is defined as a layer that is invariant under the translation transformation $\bpsi(\bth, \al) := \bth + \al \mbone_\mca$ ($\al \in \mbr$). 
For example, a linear layer immediately before the softmax layer is translation-invariant.
In the following, we derive EoM and show that its theoretical prediction of decay rates dramatically matches empirical results, indicating the importance of the counter term.
In Appendix \ref{app: Equation of Motion for bthapara}, we additionally discuss the differences between GF and GD in translation-invariant layers. 

For convenience, we first decompose $\bth_\mca$ to two vectors (Figure \ref{fig: dynamics of trln. inv. layers}); $\bthperp$ is orthogonal to $\nabla f(\bth)$, and $\bthpara$ is orthogonal to $\bthperp$. Here, note that $\nabla f(\bth)$ is orthogonal to $\mbone_\mca$ because of translation invariance; in fact, differentiating both sides of $f (\bth + \al \mbone_\mca) = f(\bth)$ with respect to $\al$ and setting $\al = 0$, we have $ \mbone_\mca \cdot \nabla f (\bth) = 0$ (see also Lemma \ref{lem: gradient constraint and relation of trln-inv. layers} in Appendix \ref{app: Proof of thm: EoM of bthperp (translation)}). 
Formally, we define $\bthaperp$, $\bthapara$, and the projection matrix $P$ as
$\bthperp := P \btha = \f{\mbone_\mca \cdot \btha}{d_\mca} \mbone_\mca$, 
$\bthpara := (I - P) \btha = \btha - \bthperp$, and
$P := \f{1}{d_\mca} \mbone_\mca \mbone_\mca^\top$,
where $d_\mca$ is the dimension of $\mca$.

We construct the EoM for $\bthaperp$ (the EoM for $\bthapara$ is given in Appendix \ref{app: Equation of Motion for bthapara} for completeness). 
\begin{theorem}[EoM for $\bthperp$] \label{thm: EoM of bthperp (translation)}
EoM (\ref{eq: gradient flow}) gives $\dot{\bth}_{\mca\perp} (t) = -\la \bthperpt - \eta P\bx(\btht)$.
Specifically, this is equivalent to 
$\dot{\bth}_{\mca\perp} (t) = -\la \bthperpt 
 \Longleftrightarrow 
\bthperpt = \bthperp(0) e^{-\la t}$
for $\bx = \bm{0}$ (GF) and
$\dot{\bth}_{\mca\perp} (t) = - (\la + \f{\eta\la^2}{2}) \bthperpt
\Longleftrightarrow  
\bthperpt = \bthperp(0) e^{- (\la + \f{\eta\la^2}{2}) t}$
for $\bx = \tbx_0$ (EoM).
\end{theorem}
The proof is based on Equations (\ref{eq: gradient flow}, \ref{eq: hat xi0}) and is given in Appendix \ref{app: Proof of thm: EoM of bthperp (translation)}.
$\bthaperp$ monotonically collapses to zero as $t\rightarrow \infty$ in either case of $\bx = \bm{0}$ or $\bx \neq \bm{0}$; thus,
as $t$ increases, the dynamics is restricted onto the subspace orthogonal to $\bthaperp$ (Figure \ref{fig: dynamics of trln. inv. layers}).
The decay rate is corrected by the counter term from $\eta \la$ to $\eta \la + \f{\eta^2\la^2}{2}$, as is also done for $r_\mca$ in Section \ref{sec: Learning Dynamics of Scale-invariant Layers}.
Therefore, the $\bthperp$ of GD decays faster than that of GF.
Figure \ref{fig: decays of bthaperp} and Table \ref{tab: decay rates of bthaperp} support our findings. In particular, Table \ref{tab: decay rates of bthaperp} shows that the decay rates predicted by EoM dramatically match those of GD, indicating the importance of the counter term.

\begin{figure}[htbp]
    \begin{minipage}[t]{0.4\linewidth}
        \centering
        \includegraphics[width=0.85\linewidth]
        {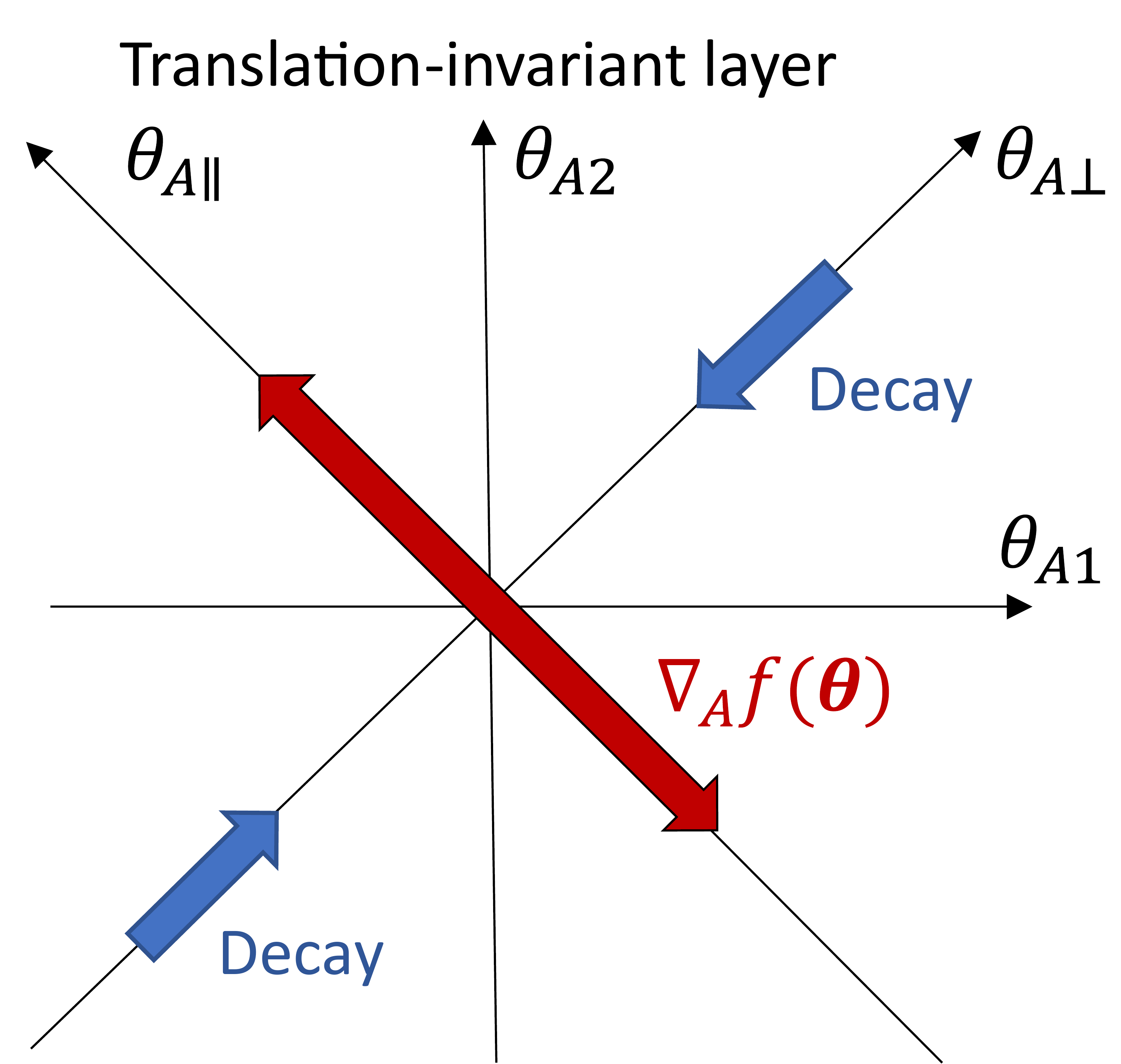}
        \caption{\textbf{Learning dynamics of translation-invariant layer.} Here, $\btha = (\th_{\mca 1}, \th_{\mca 2})^\top$. 
        $\bthaperp$ decays to $\bm{0}$ (also shown in Figure \ref{fig: decays of bthaperp}). The decay of GD is faster than that of GF (Theorem \ref{thm: EoM of bthperp (translation)}). 
        As $t$ increases, the dynamics is restricted onto the subspace orthogonal to $\bthaperp$. 
        }
        \label{fig: dynamics of trln. inv. layers}
    \end{minipage} \,
    \begin{minipage}[t]{0.6\linewidth}
        \centering
        \includegraphics[width=0.9\linewidth]
        {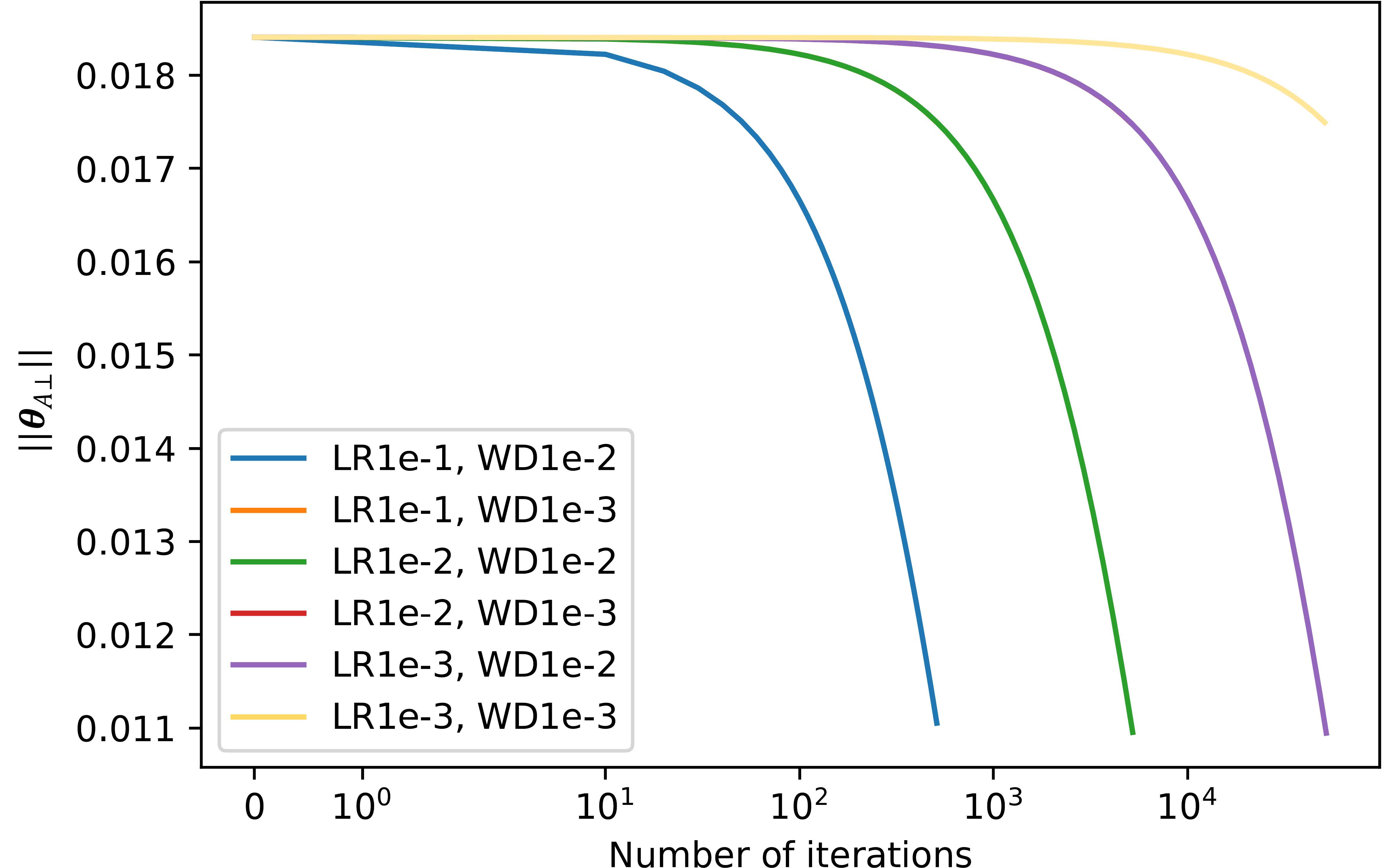}
        \caption{\textbf{Decay of $||\bthaperp||$ (GD).} $||\bthaperp||$ monotonically decays to zero, as suggested by Theorem \ref{thm: EoM of bthperp (translation)}. $\mca$ is translation-invariant layer. LR and WD mean learning rate and weight decay, respectively. Note that the orange and green curves (LR1e-1, WD1e-3 and LR1e-2, WD1e-2) and the red and purple curves (LR1e-2, WD1e-3 and LR1e-3, WD1e-2) totally overlap. The decay rates of all curves are given in Table \ref{tab: decay rates of bthaperp}. See Section \ref{sec: Experiment} for experimental settings.}
        \label{fig: decays of bthaperp}
    \end{minipage}
\end{figure}

\begin{table}[htbp]
\caption{\textbf{Decay rates of $||\bthaperp||$.} The theoretical predictions by EoM (third column) dramatically match experimental results of GD (fourth column) much better than GF (second column), indicating the importance of the counter term. LR and WD mean learning rate and weight decay, respectively. The colors correspond to those in Figure \ref{fig: decays of bthaperp}. See Section \ref{sec: Experiment} for experimental settings.}
\label{tab: decay rates of bthaperp}
\centering
\begin{tabular}{llll}
    \hline
    (LR, WD) & Theory (GF) & Theory (EoM: \textbf{Ours}) & Experiment (GD)\\
    \hline \hline 
    ($10^{-1}, 10^{-2}$) ({\color{blue} blue}) & $10^{-3}$ & $1.0005 \times 10^{-3}$ & $1.0005003484995967 \times 10^{-3}$ \\
    ($10^{-1}, 10^{-3}$) ({\color{orange} orange}) & $10^{-4}$ & $1.00005 \times 10^{-4}$ & $1.0000500182363355 \times 10^{-4}$ \\
    ($10^{-2}, 10^{-2}$) ({\color{teal} green}) & $10^{-4}$ & $1.00005 \times 10^{-4}$ & $1.0000499809795858 \times 10^{-4}$ \\
    ($10^{-2}, 10^{-3}$) ({\color{red} red}) & $10^{-5}$ & $1.000005 \times 10^{-5}$ & $1.0000049776814671 \times 10^{-5}$ \\
    ($10^{-3}, 10^{-2}$) ({\color{violet} purple}) & $10^{-5}$ & $1.000005 \times 10^{-5}$ & $1.0000050475312426 \times 10^{-5}$ \\
    ($10^{-3}, 10^{-3}$) ({\color{yellow} yellow}) & $10^{-6}$ & $1.0000005 \times 10^{-6}$ & $1.0000005475009833 \times 10^{-6}$ \\
    \hline
\end{tabular}
\end{table}

\section{Experiment} \label{sec: Experiment}
We explain our experimental settings for Figures \ref{fig: th vs ex discerr}--\ref{fig: decays of bthaperp} and Table \ref{tab: decay rates of bthaperp}.
Our network consists of a first linear layer, swish activation \cite{ramachandran2017swish_orginal_paper}, second linear layer, batch normalization \cite{ioffe2015batch_norm_original_paper}, third linear layer, and last softmax layer. Cross-entropy is used for the loss function. We note that the second linear layer is scale-invariant, and the last linear layer is translation invariant. The batch normalization uses fixed statistics to keep the scale invariance of the second linear layer. Swish is chosen to ensure differentiability. None of the linear layers have a bias term. 
The dataset is the training set of MNIST \cite{MNIST}, and thus, the batch size is 60,000. 
Gradient descent is used for the optimizer.
We use 64-bits of precision for all computations.
To simulate GF and EoM, we use a sufficiently small learning rate ($10^{-5}$).
The results are produced from only one random seed to save on computational costs, but we confirm that different random seeds lead to similar results.
More detailed information is given in Appendix \ref{app: Details of Experiment} and our code.
In all experiments, we use $\bx = \tbx_0$ for EoM. We do not include higher-order counter terms, such as $\tbx_1$, because they require third and higher order derivatives of the loss function and are thus extremely memory-consuming. We could circumvent this issue, e.g., by applying Hessian-free optimization \cite{martens2010hessian-free_optimization}, but this is out of our current scope.

\section{Conclusion and Limitations} \label{sec: Conclusion and Limitations}
In this work, to fill the critical gap between GF and GD, we add a counter term to GF and obtain EoM, a continuous differential equation that precisely describes the discrete learning dynamics of GD.
To show to what extent GF and EoM are precise in describing GD's discrete dynamics, we derive the leading order of discretization error, as is often missed in the literature on the continuous approximation of discrete GD algorithms.
We further derive a sufficient condition for learning rates for the discretization error to be small.
We apply our theory to two specific cases, scale- and translation-invariant layers, indicating the importance of the counter term for a better description of the discrete learning dynamics of GD.
Our experimental results support our theoretical findings.
Throughout this paper, we focus only on GD and GF to expose the ideas simply, and our study does not include stochasticity (e.g., SGD and SDE), acceleration methods (e.g., momentum and Nesterov \cite{Nesterov}), or adaptive optimizers (e.g., Adam \cite{Adam}). 
Nonetheless, they could be combined with our analysis, for example, using error analysis of SDEs \cite{li2019stochastic_modified_equations_math_found_JMLR2019, feng2020uniform}, continuous-time accelerated methods \cite{krichene2015accelerated, su2016differential, krichene2017acceleration, xu2018accelerated,xu2018continuous_SDE_is_good_for_SGD, aujol2019optimal, kovachki2021continuous}, and continuous-time Adam \cite{barakat2021continuous_Adam}. 
See Appendix \ref{app: Supplementary Discussion} for more discussions.
Therefore, our study could be extended to import continuous analysis to the discrete analysis of various GD algorithms.
In this sense, our work bridges discrete and continuous analyses of GD algorithms.




\section*{Acknowledgment}
We thank Shuhei M. Yoshida for his insightful comments on the dynamics of scale-invariant layers and the experimental settings.
We also thank Hidenori Tanaka for his discussion that inspired us to start this study.

\bibliographystyle{unsrt}
\bibliography{References} 

\section*{Checklist}
\begin{enumerate}

\item For all authors...
\begin{enumerate}
  \item Do the main claims made in the abstract and introduction accurately reflect the paper's contributions and scope?
    \answerYes{}
  \item Did you describe the limitations of your work?
    \answerYes{} See Sections \ref{sec: Experiment} and \ref{sec: Conclusion and Limitations} and Appendix \ref{app: Supplementary Discussion}.
  \item Did you discuss any potential negative societal impacts of your work?
    \answerNA{}
  \item Have you read the ethics review guidelines and ensured that your paper conforms to them?
    \answerYes{}
\end{enumerate}

\item If you are including theoretical results...
\begin{enumerate}
  \item Did you state the full set of assumptions of all theoretical results?
    \answerYes{}
  \item Did you include complete proofs of all theoretical results?
    \answerYes{} See Appendix \ref{app: Proofs}.
\end{enumerate}

\item If you ran experiments...
\begin{enumerate}
  \item Did you include the code, data, and instructions needed to reproduce the main experimental results (either in the supplemental material or as a URL)?
    \answerYes{} See the code in the supplemental material.
  \item Did you specify all the training details (e.g., data splits, hyperparameters, how they were chosen)?
    \answerYes{} See Section \ref{sec: Experiment} and Appendix \ref{app: Details of Experiment}.
  \item Did you report error bars (e.g., with respect to the random seed after running experiments multiple times)?
    \answerNo{} To save computational costs, we do not run experiments with multiple random seeds, but we confirm that different random seeds give similar results, as stated in Section \ref{sec: Experiment}.
  \item Did you include the total amount of compute and the type of resources used (e.g., type of GPUs, internal cluster, or cloud provider)?
    \answerYes{} See Appendix \ref{app: Details of Experiment}.
\end{enumerate}

\item If you are using existing assets (e.g., code, data, models) or curating/releasing new assets...
\begin{enumerate}
  \item If your work uses existing assets, did you cite the creators?
    \answerYes{} See our code.
  \item Did you mention the license of the assets?
    \answerYes{} See our code.
  \item Did you include any new assets either in the supplemental material or as a URL?
    \answerYes{} See our code.
  \item Did you discuss whether and how consent was obtained from people whose data you're using/curating?
    \answerNA{} We do not use such data.
  \item Did you discuss whether the data you are using/curating contains personally identifiable information or offensive content?
    \answerNA{} The data we are using do not include such information.
\end{enumerate}

\item If you used crowdsourcing or conducted research with human subjects...
\begin{enumerate}
  \item Did you include the full text of instructions given to participants and screenshots, if applicable?
    \answerNA{}
  \item Did you describe any potential participant risks, with links to Institutional Review Board (IRB) approvals, if applicable?
    \answerNA{}
  \item Did you include the estimated hourly wage paid to participants and the total amount spent on participant compensation?
    \answerNA{}
\end{enumerate}

\end{enumerate}

\appendix
\input{app_proofs}

\input{app_broken_symmetry}

\input{app_effectiveLR}
\input{app_bthapara}
\input{app_details_of_experiment}
\input{app_suppl_experiment}

\input{app_suppl_discussion}

\end{document}

%% file: Abstract.tex
We derive and solve an ``Equation of Motion'' (EoM) for deep neural networks (DNNs), a differential equation that precisely describes the discrete learning dynamics of DNNs.
Differential equations are continuous but have played a prominent role even in the study of discrete optimization (gradient descent (GD) algorithms).
However, there still exist gaps between differential equations and the actual learning dynamics of DNNs due to \textit{discretization error}. 
In this paper, we start from gradient flow (GF) and derive a counter term that cancels the discretization error between GF and GD.
As a result, we obtain \textit{EoM}, a continuous differential equation that precisely describes the discrete learning dynamics of GD.
We also derive discretization error to show to what extent EoM is precise.
In addition, we apply EoM to two specific cases: scale- and translation-invariant layers.
EoM highlights differences between continuous-time and discrete-time GD, indicating the importance of the counter term for a better description of the discrete learning dynamics of GD.
Our experimental results support our theoretical findings.

%% file: app_proofs.tex
\clearpage
\section*{Appendices}
\section{Proofs} \label{app: Proofs}
\subsection{Proof of Theorem \ref{thm: EoDE}} 
\label{app: Proof of thm: EoDE}
\begin{proof}
The integral form of Taylor's theorem gives
\begin{align}
    \bth (k\eta+\eta) - \bth(k\eta)
    =& \eta \dot{\bth} (k\eta) + \eta^2 \int_0^1 ds \ddot{\bth} (\eta(k+s)) (1-s) \nn
    =& - \eta \bm{g}(\bth(k\eta)) -\eta^2\bm{\xi}(\bth(k\eta)) +\eta^2\int_0^1 ds \ddot{\bth}(\eta(k+s)) (1-s) \, .
    \label{tmpeq: continuous}
\end{align}
Remember the definition of the discrete gradient descent:
\begin{align}
    \bth_{k+1} - \bth_k = - \eta \bm{g}(\bth_k) \, .
    \label{tmpeq: discrete}    
\end{align}
Subtracting Equation (\ref{tmpeq: discrete}) from Equation (\ref{tmpeq: continuous}), we have
\begin{align}
    \berr_{k+1} - \berr_k 
    &= -\eta ( \bm{g}(\bth(k\eta)) - \bm{g} (\bth_k))
        -\eta^2\bm{\xi}(\bth(k\eta)) +\eta^2\int_0^1ds\ddot{\bth}(\eta(k+s)) (1-s) \\
    &= -\eta ( \bm{g}(\bth(k\eta)) - \bm{g} (\bth(k\eta) - \berr_k))
        -\eta^2\bm{\xi}(\bth(k\eta)) +\eta^2\int_0^1ds\ddot{\bth}(\eta(k+s)) (1-s) \, . 
\end{align}
\end{proof}

\subsection{Proof of Theorem \ref{thm: Leading order of EoDE}} 
\label{app: Proof of thm: Leading order of EoDE}
\begin{proof}
The proof is by induction. For $k=0$, $\berr_0 = O(\eta^\ga)$ by assumption. If $\berr_k = O(\eta^\ga)$ for $k \geq 1$, Theorem \ref{thm: EoDE} gives
\begin{align}
    &\berr_{k+1} 
    = \berr_k - \eta ( \bg(\bth(k\eta)) - \bg(\bth(k\eta) - \berr_k) ) + \bm{\La}(\bth(k\eta)) 
    = O(\eta^\ga) + O(\eta^{\ga+1}) + O(\eta^\ga) = O(\eta^\ga) \, . 
\end{align}
$\eta ( \bg(\bth(k\eta)) - \bg(\bth(k\eta) - \berr_k) ) = O(\eta^{\ga+1})$ follows from Taylor's expansion of $\bg(\bth(k\eta) - \berr_k)$ around $\bth(k\eta)$ and from assumption $\berr_k = O(\eta^\ga)$:
\begin{align}
    - \eta ( \bg(\bth(k\eta)) - \bg(\bth(k\eta) - \berr_k) ) = \eta ( \berr_k \cdot \nabla \bg (\bthketa) + O(|| \berr_k ||^2)) 
    = O(\eta^{\ga+1}) \, .
\end{align}
\end{proof}

\subsection{Proof of Theorem \ref{thm: Solution of EoLDE}} \label{app: Proof of thm: Solution of EoLDE}
\begin{proof}
The proof of Theorem \ref{thm: Solution of EoLDE} consists of the following three Lemmas, all of which are proved in the following sections.

\begin{lemma} \label{lem: tmplem1}
\begin{align}
    \int_0^1 ds \ddot{\bth} (\eta(k+s)) (1-s) = \sum_{n=0}^\infty \f{\eta^n}{(n+2)!} \f{d^{n+2}}{dt^{n+2}} \bth (k\eta)
    \label{tmpeq: eq in tmplem1}
\end{align}
\end{lemma}

\begin{lemma} \label{lem: tmplem2}
For $n \geq 1$,
\begin{align}
    \f{d^{n}}{dt^{n}} \bth (t) 
    = (-1)^n \sum_{k_1, \cdots, k_n = 0}^{\infty} \eta^{k_1 + \cdots k_n} \mcd_{k_1} \cdots \mcd_{k_{n-1}} \Xi_{k_n}
    \, , \label{tmpeq: eq in tmplem2}
\end{align}
where $\mcd_{k_1} \dotsm \mcd_{k_{n-1}} := 1$ for $n = 1$.
\end{lemma}

\begin{lemma} \label{lem: tmplem3}
\begin{align}
    \int_0^1 ds \ddbth (\eta(k+s)) (1-s)
    = \sum_{j=0}^{\infty} \sum_{i=2}^{j+2} \sum_{k_1+\cdots+k_i = j - i + 2} \f{(-1)^i}{i!} \eta^j D_{k_1} \cdots D_{k_{i-1}} \Xi_{k_i}
    \label{tmpeq: eq in tmplem3}
\end{align}
\end{lemma}

Theorem \ref{thm: Solution of EoLDE} follows by comparing both sides of Equation (\ref{eq: EoLDE}) order-by-order with using Equation (\ref{tmpeq: eq in tmplem3}) and the expansion of $\bx$ (\ref{eq: counter term expanded}).
\end{proof}

\subsubsection{Proof of Lemma \ref{lem: tmplem1}}
\begin{proof}
\begin{align}
    & \int_{0}^1 ds \ddbth(\eta(k+s)) (1-s) \nn
    =& \f{1}{\eta^2} \int_{k\eta}^{k\eta+\eta} ds \ddbth(s) (k\eta+\eta-s) \label{tmpl: tmplem1 1} \\
    =& \f{1}{\eta^2} \int_0^\eta ds^\prime [
        \ddbth(k\eta) (\eta - s^\prime) + \dddbth(k\eta)(\eta-s^\prime)s^\prime + \f{1}{2!} \ddddot{\bth}(k\eta)(\eta-s^\prime){s^\prime}^2 + \dotsm] \label{tmpl: tmplem1 2} \\
    =& \sum_{n=0}^\infty \f{\eta^n}{(n+2)!} \f{d^{n+2}}{dt^{n+2}} \bth (k\eta) \label{tmpl: tmplem1 3}
\end{align}
From Line (\ref{tmpl: tmplem1 1}) to (\ref{tmpl: tmplem1 2}), we used $s^\prime := s - k\eta$ and the Taylor expansion of $\ddbth(k\eta + s^\prime)$ around $k\eta$. From Line (\ref{tmpl: tmplem1 2}) to (\ref{tmpl: tmplem1 3}), we used $\int_0^\eta ds^\prime (\eta - s^\prime) {s^\prime}^n = \f{\eta^{n+2}}{(n+1)(n+2)}$ for $n \geq 0$.
\end{proof}

\subsubsection{Proof of Lemma \ref{lem: tmplem2}}
\begin{proof}
Note that given $\dbtht = - \bg(\btht) - \eta \bx(\btht)$, we have
\begin{align}
    \f{d}{dt}\left(\f{d^{n-1}}{dt^{n-1}} \btht\right) = - \mcd\left( \f{d^{n-1}}{dt^{n-1}} \btht \right) 
    \,\,\,\,\, (n \geq 1) \, ,
\end{align}
where $d^0 \bth/ dt^0 := \bth$.
Therefore,
\begin{align}
    \f{d^n}{dt^n}\btht 
    = (-1)^{n-1} \mcd^{n-1} (-\bg - \eta \bx) 
    = (-1)^n \mcd^{n-1} \Xi
    \,\,\,\,\, (n \geq 1) \, .
\end{align}
Thus, by definition of $\mcd$, $\mcd_\al$, and $\Xi_\al$ (Theorem \ref{thm: Solution of EoLDE} in Section \ref{sec: Solution to EoLDE}), we have
\begin{align}
    \f{d^n}{dt^n}\btht 
    =& (-1)^n (\sum_{k_1=0}^\infty \eta^{k_1}\mcd_{k_1}) \cdots (\sum_{k_{n-1}=0}^{\infty} \eta^{k_{n-1}} \mcd_{k_{n-1}}) \Xi \\
    =& (-1)^n \sum_{k_1, \cdots, k_n = 0}^{\infty} \eta^{k_1 + \cdots k_n} \mcd_{k_1} \cdots \mcd_{k_{n-1}} \Xi_{k_n} \, .
\end{align}
\end{proof}

\subsubsection{Proof of Lemma \ref{lem: tmplem3}}
\begin{proof}
From Lemma \ref{lem: tmplem1} and \ref{lem: tmplem2}, we have
\begin{align}
    & \int_{0}^1 ds \ddbth(\eta(k+s)) (1-s) \nn
    =& \sum_{n=0}^\infty \f{\eta^n}{(n+2)!} \f{d^{n+2}}{dt^{n+2}} \bthketa \\
    =& \sum_{n=0}^{\infty}\f{\eta^n}{(n+2)!} (-1)^{n+2} \sum_{k_1,\cdots,k_{n+2}=0}^{\infty} \eta^{k_1+\cdots+k_{n+2}} \mcd_{k_1}\cdots \mcd_{k_{n+1}}\Xi_{k_{n+2}} \\
    =& \sum_{n=0}^{\infty} \sum_{k_1,\cdots,k_{n+2}=0}^{\infty} \f{(-1)^{n}}{(n+2)!} \eta^{n + k_1 + \cdots + k_{n+2}} \mcd_{k_1} \cdots \mcd_{k_{n+1}} \Xi_{k_{n+2}} \\
    =& \sum_{j=0}^\infty \sum_{i=2}^{j+2} \sum_{k_1 + \cdots + k_i = j - i + 2} \f{(-1)^{i-2}}{i!} \eta^j \mcd_{k_1} \cdots \mcd_{k_{i-1}} \Xi_{k_i} \, .
\end{align}
On the last line, we replaced $n+2$ and $n+k_1 + \cdots + k_{n+2}$ with $i$ and $j$, respectively.
\end{proof}

\subsection{Proof of Corollary \ref{cor: Discretization error of bx}} \label{app: Proof of cor: Discretization error of bx}
\begin{proof}
By assumption, we use
\begin{align}
    \bx(\bth) = \eta^2 \sum_{\al=0}^{\ga-1} \eta^\al \tbx_\al \, .
\end{align}
From Theorem \ref{thm: Solution of EoLDE}, we have
\begin{align}
    \bm{\La}(\bth) 
    &= \eta^2 \int_0^1 ds \ddot{\bth}(\eta(k+s)) (1-s) 
        -\eta^2\bm{\xi}(\bth(k\eta)) \\
    &= \eta^2 \sum_{\al=0}^{\infty} \eta^\al \tbx_\al
        - \eta^2 \sum_{\al=0}^{\ga-1} \eta^\al \tbx_\al \\
    &= \eta^2 \sum_{\al=\ga}^{\infty} \eta^\al \tbx_\al \\
    &= \eta^{\ga+2} \tbx_\ga + O(\eta^{\ga+3}) \, .
\end{align}
Therefore, Theorem \ref{thm: Leading order of EoDE} gives
\begin{align}
    \berr_{k+1} 
    =& \berr_k + \bm{\La}(\bthketa) + O (\eta^{\ga+3}) \\ 
    =& \berr_k + \eta^{\ga+2} \tbx_\ga + O(\eta^{\ga+3}) + O(\eta^{\ga+3}) \\
    =& \berr_k + \eta^{\ga+2} \tbx_\ga + O(\eta^{\ga+3}) \, .
\end{align}
\end{proof}

\subsection{Proof of Corollary \ref{cor: Learning rate bound when bx=0}} \label{app: Proof of cor: Learning rate bound when bx=0}
\begin{proof}
From Equation (\ref{eq: disc. err. with no xi}), we have
\begin{align}
    \berr_k = \berr_0 + \sum_{s=0}^{k-1} \f{\eta^2}{2} (H(\bth(s\eta) + \la I ) \bg(\bth(s\eta)) + O(\eta^3) \, . 
\end{align}
Because $\berr_0 = O(\eta^3)$ by assumption, we have
\begin{align}
    &\berr_k 
    = \sum_{s=0}^{k-1} \f{\eta^2}{2} (H(\bth(s\eta) + \la I ) \bg(\bth(s\eta)) + O(\eta^3) \\
    \therefore\,\,\,& ||\berr_k|| \leq \f{\eta^2}{2} \sum_{s=0}^{k-1} || (H(\bth(s\eta) + \la I ) \bg(\bth(s\eta)) || + O(\eta^3) \\
    & \,\,\,\,\,\,\,\,\,\,\,\,\, \leq \f{\eta^2 k}{2}  \underset{0\leq s \leq k-1}{\mr{\max}} \{ 
            || (H(\bth(s\eta) + \la I ) \bg(\bth(s\eta)) ||
        \} + O(\eta^3) \, .
\end{align}
Let $t >0$ be a given arbitrary number. Then, for $k \in \{1,2,...,\lfloor \f{t}{\eta} \rfloor\}$,
\begin{align}
    ||\berr_k|| \leq \f{\eta^2 k}{2} \underset{0\leq t^\prime \leq t}{\mr{\max}} \{ 
            || (H(\bth(t^\prime) + \la I ) \bg(\bth(t^\prime)) ||
        \} + O(\eta^3) \, .
\end{align}
Therefore, if $\eta < \sqrt{\ep / k} \sqrt{2 / {\mr{\max}}_{0\leq t^\prime \leq t} \{|| (H(\bth(t^\prime) + \la I ) \bg(\bth(t^\prime)) ||\}}$, then
\begin{align}
    || \berr_k || < \ep + O(\ep^{3/2})\, .
\end{align}
\end{proof}

\subsection{Proof of Corollary \ref{cor: Learning rate bound when bx=tbx0}} \label{app: Proof of cor: Learning rate bound when bx=tbx0}
\begin{corollary}[Learning rate bound when $\bx=\tbx_0$] \label{cor: Learning rate bound when bx=tbx0}
Let $\bx=\tbx_0$ and assume that $\berr_0=O(\eta^4)$. Let $\ep$ and $t$ be arbitrary positive numbers. If the step size satisfies
\begin{align}
        \eta < \sqrt[3]{\f{\ep}{k}} \sqrt[3]{
            \f{12}{
                \underset{0\leq t^\prime \leq t}{\max} \{  || 4 (H(\bth(t^\prime)) + \la I)^2 \bg(\bth(t^\prime)) 
                +  \bg(\bth(t^\prime))^\top \nabla H(\bth(t^\prime)) \bg(t^\prime) || \}
            }}
    \, ,
\end{align}
for some $k \in \{1,2,..., \lfloor\f{t}{\eta}\rfloor \}$, then the discretization error can be arbitrarily small:
\begin{align}
    ||\berr_k|| < \ep + O(\ep^{\f{4}{3}}) \, .
\end{align}
\end{corollary}

\begin{proof}
From Equation (\ref{eq: hat xi1}) and Corollary \ref{cor: Discretization error of bx} and by assumption, we have
\begin{align}
    \berr_k = \berr_0 + \eta^3 \sum_{s=0}^{k-1}\{ 
        \f{1}{2}(\tbx_0(\bth(s\eta)\cdot\nabla))\bg(\bth(s\eta)) + \f{1}{6}(\bg(\bth(s\eta))\cdot\nabla)\tbx_0(\bth(s\eta))
    \} + O(\eta^4) \,.
\end{align}
Because $\berr_0 = O(\eta^4)$ by assumption, we have
\begin{align}
    \berr_k 
    &= \eta^3 \sum_{s=0}^{k-1}\{ 
        \f{1}{2}(\tbx_0(\bth(s\eta)\cdot\nabla))\bg(\bth(s\eta)) + \f{1}{6}(\bg(\bth(s\eta))\cdot\nabla)\tbx_0(\bth(s\eta))
    \} + O(\eta^4) \\
    &= \eta^3 \sum_{s=0}^{k-1}\{ 
        \f{1}{3}(H(\bth(s\eta))+\la I)^2\bg(\bth(s\eta)) + \f{1}{12} \bg^\top(\bth(s\eta)) \nabla H(\bth(s\eta)) \bg(\bth(s\eta))
    \} + O(\eta^4) \, .
\end{align}
Therefore,
\begin{align}
    ||\berr_k|| 
    &\leq \eta^3 \sum_{s=0}^{k-1} || 
        \f{1}{3}(H(\bth(s\eta))+\la I)^2\bg(\bth(s\eta)) + \f{1}{12} \bg^\top(\bth(s\eta)) \nabla H(\bth(s\eta)) \bg(\bth(s\eta))
    || + O(\eta^4) \\
    &\leq \f{\eta^3 k}{12} \underset{0\leq s\leq k-1}{\mr{max}} \{ || 
        4(H(\bth(s\eta))+\la I)^2\bg(\bth(s\eta)) + \bg^\top(\bth(s\eta)) \nabla H(\bth(s\eta)) \bg(\bth(s\eta))
    ||\} \nn 
    &\,\,\,\, + O(\eta^4) \, .
\end{align}
Let $t >0$ be a given arbitrary number. Then, for $k \in \{1,2,...,\lfloor \f{t}{\eta} \rfloor\}$,
\begin{align}
    ||\berr_k|| 
    &\leq  \f{\eta^3 k}{12} \underset{0\leq t^\prime \leq t}{\mr{max}} \{ || 
        4(H(\bth(t^\prime ))+\la I)^2\bg(\bth(t^\prime)) + \bg^\top(\bth(t^\prime)) \nabla H(\bth(t^\prime)) \bg(\bth(t^\prime))
    ||\} + O(\eta^4) \, .
\end{align}
Therefore, if
\begin{align}
    \eta < \sqrt[3]{\f{\ep}{k}} \sqrt[3]{
        \f{12}{
            \underset{0\leq t^\prime \leq t}{\max} \{  || 4 (H(\bth(t^\prime)) + \la I)^2 \bg(\bth(t^\prime)) 
            +  \bg(\bth(t^\prime))^\top \nabla H(\bth(t^\prime)) \bg(t^\prime) || \}
        }} \, ,
\end{align}
then $||\berr_k|| < \ep + O(\ep^{4/3})$.
\end{proof}

\subsection{Proof of Theorem \ref{thm: EoM of r}} \label{app: Proof of thm: EoM of r}
We use the following Lemmas.
\begin{lemma} \label{lem: gradient constraints for scale-invariant layers}
For scale-invariant layers $\mca$, the following equations hold:
\begin{align}
    & \btha \cdot\nabla f(\bth) = \btha \cdot\nabla_\mca f(\bth) = 0 \\
    & H_\mca(\bth) \btha + \nabla_\mca f(\bth) = 0 \\
    & \nabla_\mcac \nabla_\mca^\top f(\bth) \btha = 0 \, ,
\end{align}
where $H_\mca (\bth) := (\mbone_\mca \odot \nabla) (\mbone_\mca \odot \nabla)^\top f(\bth)$.
\end{lemma}
\begin{proof}
Differentiating both sides of $f(\al_\mca \odot \bth) = f(\bth)$ with respect to $\al$, we have
\begin{align}
    \btha \cdot \nabla f(\bth) = \btha \cdot \nabla_\mca f (\al_\mca \odot \bth) = 0 \, ,
\end{align}
where $\nabla_\mca f(\al_\mca \odot \bth)$ means $(\nabla_\mca f(\bth))|_{\bth=\al_\mca \odot \bth}$. For $\al=1$, we have
\begin{align}
    \btha \cdot \nabla f (\bth) = \btha \cdot \nabla_\mca f (\bth) = 0 \, .
\end{align}
Applying $\nabla$, we have 
\begin{align}
    &(\btha \cdot \nabla_\mca) \nabla f (\bth) + \nabla_\mca f (\bth) = 0 \\
    \Longleftrightarrow & (\btha\cdot \nabla_\mca) (\nabla_\mca + \nabla_\mcac) f(\bth) + \nabla_\mca f (\bth) = 0 \\
    \Longleftrightarrow & H_\mca (\bth) \btha + \nabla_\mcac \nabla_\mca^\top f(\bth) \btha + \nabla_\mca f (\bth) = 0 \, .
\end{align}
Multiplying by $\mbone_\mcac \odot$, we have
\begin{align}
    \nabla_\mcac \nabla_\mca^\top f (\bth) \btha = 0 \, .
\end{align}
Therefore,
\begin{align}
    H_\mca (\bth) \btha + \nabla_\mca f (\bth) = 0 \, .
\end{align}
\end{proof}

\begin{lemma} \label{lem: gradient relation of scale-invariant layers}
For scale-invariant layers $\mca$, the following equations hold:
\begin{align}
    \nabla_\mca f (\bth) = \f{1}{r_\mca} \nabla_\mca f(\hbtha + \bthac) \, ,
\end{align}
where $\nabla_\mca f (\hbtha + \bthac) := (\nabla_\mca f (\bth))|_{\bth= \hbtha + \bthac}$.
\end{lemma}
\begin{proof}
Note that $f(\bth) = f(\al_\mca \odot \bth) = f(\al \bth_\mca + \bth_{\mca^\mathsf{c}})$.
Differentiating both sides with respect to $\bth$, we have
\begin{align}
    &\nabla f(\bth)\\
    =& \nabla(f(\al_\mca \odot \bth)) \\
    =& (\nabla_\mca + \nabla_\mcac) (f(\al_\mca \odot \bth)) \\
    =& \al \nabla_\mca f (\al_\mca \odot \bth) 
    + \nabla_{\mca^\mathsf{c}} f (\al_\mca \odot \bth) \, . \label{eq: gradient relation of scale-inv. layer}
\end{align}
For $\al = 1 / r_\mca$, we have
\begin{align}
    \nabla f(\bth) 
    =& \f{1}{r_\mca} \nabla_\mca f (\hbtha + \bthac) 
    + \nabla_{\mca^\mathsf{c}} f (\hbtha + \bthac) \, .
\end{align}
Therefore,
\begin{align}
    \nabla_\mca f(\bth) = \mbone_\mca \odot \nabla f (\bth)
    =& \mbone_\mca \odot ( \f{1}{r_\mca} \nabla_\mca f (\hbth_\mca + \bthac) 
        + \nabla_{\mca^\mathsf{c}} f (\hbtha + \bth_{\mca^\mathsf{c}}) ) \\
    =& \f{1}{r_\mca} \nabla_\mca f (\hbth_\mca + \bthac) \, .
\end{align}
\end{proof}
\begin{lemma} \label{lem: hessian relation of scale-invariant layers}
For scale-invariant layers $\mca$, the following equations hold for all $\al > 0$:
\begin{align}
    &H(\bth) = \al^2 H_\mca (\al_\mca \odot \bth) + \al (\nabla_\mcac \nabla_\mca^\top f (\al_\mca \odot \bth) + \nabla_\mca \nabla_\mcac^\top f (\al_\mca \odot \bth) ) + H_\mcac (\al_\mca \odot \bth) \\
    &H(\bth) \btha = \al^2 H_\mca (\al_\mca \odot \bth) \bth_\mca \\
    &H(\bth) \btha = H_\mca (\bth) \btha
    \, ,
\end{align}
where $H_\mca (\al_\mca \odot \bth) := ((\mbone_\mca \odot \nabla) (\mbone_\mca \odot \nabla)^\top f (\bth)) |_{\bth = \al_\mca \odot \bth}$.
\end{lemma}
\begin{proof}
Because $\nabla f (\bth) = \al \nabla_\mca f(\al \bth_\mca) + \nabla_{\mca^\mathsf{c}} f(\bth_{\mca^\mathsf{c}})$ (Equation \ref{eq: gradient relation of scale-inv. layer}), 
\begin{align}
    H(\bth) 
    &= \nabla \nabla^\top f (\bth) \\
    &= \nabla (\al \nabla_\mca^\top f (\al_\mca \odot \bth) + \nabla_\mcac^\top f (\al_\mca \odot \bth) ) \\
    &= (\nabla_\mca + \nabla_\mcac) (\al \nabla_\mca^\top f (\al_\mca \odot \bth) + \nabla_\mcac^\top f (\al_\mca \odot \bth)) \\
    &= \al^2 H_\mca (\al_\mca \odot \bth) + \al (\nabla_\mcac \nabla_\mca^\top f (\al_\mca \odot \bth) + \nabla_\mca \nabla_\mcac^\top f (\al_\mca \odot \bth)  ) + H_\mcac (\al_\mca \odot \bth) \, .
\end{align}
Therefore,
\begin{align}
    H(\bth) \btha 
    &= \al^2 H_\mca(\al_\mca \odot \bth) \btha + \al \nabla_\mcac \nabla_\mca^\top f (\al_\mca \odot \bth) \btha \\
    &= \al^2 H_\mca (\al_\mca \odot \bth) \btha \, .
\end{align}
For $\al=1$, we have
\begin{align}
    H(\bth) \btha = H_\mca (\bth) \btha \, .
\end{align}
\end{proof}

We now prove Theorem \ref{thm: EoM of r}.
\begin{proof}
We use Lemmas \ref{lem: gradient constraints for scale-invariant layers}, \ref{lem: gradient relation of scale-invariant layers}, and \ref{lem: hessian relation of scale-invariant layers}.
\begin{align}
    \dot{r_\mca^2}(t) 
    =& 2 \bthat \cdot \dbth_\mca(t) \\
    =& 2 \bth_\mca(t) \cdot (-\nabla_\mca f(\btht) - \la \btha(t) - \eta \bx(\btht)  ) \\
    =& -2\la r_\mca^2 (t) - 2 \eta \btha(t) \cdot \bx(\btht) \, . \label{tmpeq: reom}
\end{align}
For $\bx = \bm{0}$, 
\begin{align}
    \dot{r_\mca^2}(t)     
    =& -2\la r_\mca^2 (t) \, .
\end{align}
For $\bx = \tbx_0$, 
\begin{align}
    \dot{r_\mca^2}(t)     
    =& -2\la r_\mca^2 (t) - 2 \eta \btha(t) \cdot \tbx_0(\btht) \\
    =& -2\la r_\mca^2 (t) - \eta (\la^2 r_\mca^2 (t) - || \nabla_\mca f(\btht) ||^2  ) \\
    =& -2 \la ( 1 + \f{\et\la}{2} ) r_\mca^2 (t) + \f{\eta}{r_\mca^2(t)} || \nabla_\mca f(\hbth_\mca(t)) ||^2 \, .
\end{align}
We used 
\begin{align}
    \btha \cdot \tbx_{0\mca} 
    &= \f{1}{2} \btha \cdot ( H(\bth) + \la I ) (\nabla f (\bth) + \la \bth) \\
    &= \f{1}{2} \btha \cdot ( H(\bth) \nabla f (\bth) + \la H(\bth) \bth + \la \nabla f (\bth) + \la^2 \bth ) \\
    &= \f{1}{2} (\btha^\top H_\mca (\bth) \nabla_\mca f (\bth) + \la \btha^\top H_\mca (\bth) \btha + \la^2 r^2_\mca ) \\
    &= \f{1}{2} (- || \nabla_\mca f (\bth) ||^2 + \la^2 r^2_\mca ) \, . \label{eq: theta dot tbx0}
\end{align}
Using $\dot{\bm{x}}(t) = - a \bm{x} + \bm{y}(t) \Leftrightarrow \bm{x}(t) = \bm{x}(0)e^{-at} + \int_0^t d\tau e^{-a(t-\tau)} \bm{y}(\tau)$, we can show the remaining equations.

\end{proof}

\subsection{Proof of Corollary \ref{cor: r at equil.}} \label{app: Proof of cor: r at equil.}

\begin{proof}
When $\bx = \bm{0}$, $r_\mca \xrightarrow{t \rightarrow \infty} 0$ is obvious from the EoM for $r_\mca$ (Theorem \ref{thm: EoM of r}).
When $\bx = \tbx_0$, EoM is given by
\begin{align}
    \dot{r^2}_\mca(t) = - 2 \la (1+\f{\eta\la}{2})r_\mca^2(t) + \f{\eta}{r^2(t)}|| \nabla_\mca f (\hbthat + \bthac(t)) ||^2 \, .
\end{align}
At equilibrium, $\dot{r}_\mca = 0$ and $|| \nabla_\mca f (\hbtha + \bthac) || = c_*$ by assumption; thus, we have
\begin{align}
    &0 = - 2 \la (1 + \f{\eta\la}{2}) r_{\mca *}^2 + \f{\eta}{r_{\mca *}^2} c^2_* \\
    \Longleftrightarrow &r_{\mca *}^2 = \sqrt{\f{\eta}{2\la + \eta\la^2}} c_* \, .
\end{align}
\end{proof}

\subsection{Proof of Theorem \ref{thm: EoM of hattheta}} \label{app: Proof of thm: EoM of hattheta}
\begin{proof}
We use Lemmas \ref{lem: gradient constraints for scale-invariant layers}, \ref{lem: gradient relation of scale-invariant layers}, and \ref{lem: hessian relation of scale-invariant layers}:
\begin{align}
    \dhbth_\mca 
    &= \f{d}{dt}\f{\bth_\mca}{r_\mca} \\
    &= -\f{\dot{r}_\mca}{r_\mca^2} \btha + \f{1}{r_\mca} \dbth_\mca \\
    &= \f{\btha}{r_\mca^2}(\la r_\mca + \eta \hbth_\mca \cdot \bx(\bth))
        + \f{1}{r_\mca} (-\nabla_\mca f(\bth) -\la \btha - \eta \bx_\mca(\bth)) \\
    &= \f{\eta}{r_\mca} \hbth_\mca (\hbth_\mca \cdot \bx_\mca(\bth)) 
        - \f{1}{r_\mca} \nabla_\mca f(\bth) - \f{\eta}{r_\mca}\bx_\mca(\bth) \\
    &= - \f{1}{r^2_\mca} \nabla_\mca f (\hbth_\mca + \bth_\mcac) 
        + \f{\eta}{r_\mca} ((\hbth_\mca \cdot\bx_\mca(\bth))\hbth_\mca - \bx_\mca(\bth)) \, ,
    \label{tmpeq: last}
\end{align}
where $\bx_\mca := \mbone_\mca \odot \bx$. We used $\dot{r}_\mca = -\la r_\mca - \eta \hbth_\mca \cdot \bx(\bth)$ (Theorem \ref{thm: EoM of r}).
Note that $\f{\eta}{r_\mca} ((\hbtha \cdot \bx_\mca(\bth))\hbtha - \bx_\mca(\bth))$ has no $\hbtha$ component; i.e., it is orthogonal to $\hbtha$.
When $\bx = \bm{0}$, Equation (\ref{tmpeq: last}) is equivalent to $\dhbth_\mca = - \f{1}{r^2_\mca} \nabla_\mca f (\hbth_\mca)$.
When $\bx = \tbx_0$, note that from Equation (\ref{eq: theta dot tbx0}),
\begin{align}
    \btha \cdot \tbx_{0\mca} = \f{1}{2} (- || \nabla_\mca f (\bth) ||^2 + \la^2 r_\mca^2) \, . \label{tmpeq: 141}
\end{align}
Therefore, 
\begin{align}
    (\hbtha \cdot \tbx_{0\mca}) \hbtha = - \f{1}{2} \f{1}{r_\mca^3}|| \nabla_\mca f (\hbtha + \bth_\mcac)||^2 \hbtha + \f{\la^2}{2}\btha \, .
\end{align}
Also, 
\begin{align}
    \tbx_{0\mca} 
    &= \f{1}{2} \mbone_\mca \odot (H(\bth) \nabla f (\bth) + \la H (\bth)\bth + \la \nabla f (\bth) + \la^2 \bth) \\
    &= \f{1}{2} (\mbone_\mca \odot H(\bth) \nabla f (\bth) + \la \mbone_\mca \odot H (\bth) (\btha + \bthac) + \la \nabla_\mca f (\bth) + \la^2 \btha  ) \\
    &= \f{1}{2} (\mbone_\mca \odot H(\bth) \nabla f (\bth) + 
        \la H_\mca (\bth) \btha + \la \nabla_\mca \nabla_\mcac^\top f (\bth) \bthac 
        + \la \nabla_\mca f (\bth) + \la^2 \btha  ) \\
    &= \f{1}{2} ( \mbone_\mca \odot H(\bth) \nabla f (\bth) + \la \nabla_\mca \nabla_\mcac^\top f (\bth) \bthac + \la^2\btha  ) \, . \label{tmpeq: 146}
\end{align}
Therefore,
\begin{align}
    &(\hbtha \cdot \tbx_{0\mca}(\bth)) \hbtha - \tbx_{0\mca} \\
    =& - \f{1}{2} \f{1}{r_\mca^3}|| \nabla_\mca f (\hbtha + \bthac) ||^2 \hbtha 
        + \f{\la^2}{2} \btha 
        - \f{1}{2} ( \mbone_\mca \odot H(\bth) \nabla f (\bth) + \la \nabla_\mca \nabla_\mcac^\top f (\bth) \bthac + \la^2 \btha )\\
    =& - \f{1}{2} \f{1}{r_\mca^3} ||\nabla_\mca f (\hbtha + \bthac) ||^2 \hbtha - \f{1}{2} \nabla_\mca \nabla^\top f (\bth) \nabla f (\bth) - \f{\la}{2} \nabla_\mca \nabla_\mcac^\top f (\bth) \bthac \\
    =& - \f{1}{2} \f{1}{r_\mca^3} ||\nabla_\mca f (\hbtha + \bthac) ||^2 \hbtha 
        - \f{1}{2} H_\mca (\bth) \nabla_\mca f (\bth) 
        - \f{1}{2} \nabla_\mca \nabla_\mcac^\top f (\bth) \nabla_\mcac f (\bth) \nn
    &-\f{1}{2} \nabla_\mca \nabla_\mcac^\top f (\bth) \la \bthac \\
    =& - \f{1}{2} \f{1}{r_\mca^3} ||\nabla_\mca f (\hbtha + \bthac) ||^2 \hbtha
        -\f{1}{2}\f{1}{r_\mca} H_\mca \nabla_\mca f (\hbtha + \bthac) \nn
    &-\f{1}{2} \f{1}{r_\mca} \nabla_\mca \nabla_\mcac^\top f (\hbtha + \bthac) (\nabla_\mcac f (\bth) + \la \bthac) \, .
\end{align}
Hence, 
\begin{align}
    \dhbth_\mca 
    =& - \f{1}{r_\mca^2} \nabla_\mca f (\hbtha + \bthac) - \f{\eta}{2r_\mca^2}( 
            H_\mca (\bth) \nabla_\mca f (\hbth_\mca + \bthac) \nn
            &\,\,\,\,\,\,\, + \nabla_\mca \nabla_\mcac^\top f (\hbtha + \bthac)(\nabla_\mcac f (\bth) \la \bthac) 
            + \f{1}{r_\mca^2}||\nabla_\mca f (\hbtha + \bthac)||^2 \hbtha
        ) \\
    =& - \f{1}{r_\mca^2} (
            I + \f{\eta}{2} H_\mca (\bth) \nn
    &+ \f{\eta}{2} (\nabla_\mcac f (\bth) + \la \bthac) \cdot \nabla_\mcac + \f{\eta}{2} \f{1}{r_\mca^2}\hbtha \nabla_\mca^\top f(\hbtha + \bthac)
        ) \nabla_\mca f (\hbtha + \bthac) \\ 
    =& - \f{1}{r_\mca^2} (
            I 
            + \f{\eta}{2} H_\mca (\bth) 
            + \f{\eta}{2} (\nabla_\mcac f (\bth) + \la \bthac) \cdot \nabla_\mcac \nn 
    &+ \f{\eta}{2} \hbtha \nabla_\mca^\top f(\bth)
        ) \nabla_\mca f (\hbtha + \bthac)
        \, .
\end{align}
\end{proof}


\subsection{Proof of Corollary \ref{cor: angular update at equil.}} \label{app: Proof of cor: angular update at equil.}
\begin{proof}
We use Lemmas \ref{lem: gradient constraints for scale-invariant layers} and \ref{lem: gradient relation of scale-invariant layers}. 
The angular update is defined as
\begin{align}
    \cos \De(t) = \f{\bthat}{r_\mca (t)}\cdot\f{\btha(t+\eta)}{r_\mca (t+\eta)} \, .
\end{align}
We evaluate the higher order terms in $\btha(t+\eta)$ and $r_\mca (t+\eta)$. First,
\begin{align}
    \btha(t+\eta) =& \bthat + \eta \dbth_\mca (t) + \f{\eta^2}{2}\ddbth_\mca (t) + O(\eta^3) \nn
    =& \bthat - \eta \nabla f(\bthat) -\eta \la \bthat - \eta^2 \bx_\mca (\btht) + \f{\eta^2}{2} \ddbth_\mca (t)  + O(\eta^3) \, .
\end{align}
The second derivative $\ddbtht$ is given by
\begin{align}
    \ddbth_\mca 
    &= \f{d}{dt}\dbth_\mca \\
    &= \f{d}{dt} (- \nabla_\mca f (\bth) -\la \btha) + O(\eta) \\
    &= - (\dbth \cdot \nabla) \nabla_\mca f(\bth) - \la \dbth_\mca + O (\eta) \\ 
    &= \nabla_\mca \nabla^\top f (\bth) (\nabla f (\bth) + \la \bth) + \la (\nabla_\mca f (\bth) + \la \btha) + O(\eta) \\
    &= \mbone_\mca \odot H (\bth) \nabla f(\bth) + \la H_\mca (\bth) \btha + \la \nabla_\mca \nabla_\mca^\top f (\bth) \bthac + \la \nabla_\mca f (\bth) + \la^2 \btha + O(\eta) \\
    &= \mbone_\mca \odot H(\bth) \nabla f (\bth) + \la \nabla_\mca \nabla_\mcac^\top f (\bth) \bthac + \la^2 \btha + O (\eta)
    \, .
\end{align}
Therefore,
\begin{align}
    \btha (t+\eta) =& \,\, \bthat - \eta \nabla_\mca f(\btht) - \eta \la \bthat - \eta^2 \bx_\mca(\btht) \nn
    &+ \f{\eta^2}{2}(  \mbone_\mca \odot H(\btht) \nabla f (\btht) + \la \nabla_\mca \nabla_\mcac^\top f (\btht) \bthac(t) + \la^2 \bthat ) + O(\eta^3) \, .
\end{align}
Next,
\begin{align}
    r_\mca(t+\eta) = r_\mca(t) + \dot{r}_\mca(t) \eta + \f{\eta^2}{2} \ddot{r}_\mca(t) + O(\eta^3) \, .
\end{align}
Because $\dot{r}_\mca = - \la r_\mca - \eta \hbtha \cdot \bx$ (use Equation (\ref{tmpeq: reom}) and $\dot{r^2}_\mca = 2 r_\mca \dot{r}_\mca$),
\begin{align}
    r_\mca(t+\eta) = r_\mca(t) - \eta \la r_\mca(t) - \eta^2 \hbthat \cdot \bx_\mca (\btht) + \f{\eta^2}{2} \ddot{r}_\mca(t) + O(\eta^3) \, .
\end{align}
In addition, because $\ddot{r}_\mca = -\la \dot{r}_\mca + O(\eta) = \la^2 r_\mca + O(\eta)$,
\begin{align}
    r_\mca(t+\eta) = r_\mca(t) - \eta \la r_\mca(t) -\eta^2 \hbthat \cdot \bx(\btht) + \f{\eta^2}{2} \la^2 r_\mca(t) + O(\eta^3) \, . 
\end{align}
Therefore, 
\begin{align}
    &\cos \De(t) = \f{\bthat}{r_\mca(t)}\cdot\f{\btha(t+\eta)}{r_\mca(t+\eta)} \\
    =& \f{\bthat}{r_\mca(t)} \cdot 
            \Big(\bthat - \eta \nabla_\mca f(\btht) - \eta \la \bthat - \eta^2 \bx_\mca(\btht) \nn 
            &+ \f{\eta^2}{2}( \mbone_\mca \odot H(\btht) \nabla f (\btht) + \la \nabla_\mca \nabla_\mcac^\top f (\btht) \bthac(t) + \la^2 \bthat ) \Big)        
            / \Big(r_\mca(t) - \eta \la r_\mca(t) \nn
            &-\eta^2 \hbthat \cdot \bx(\btht) + \f{\eta^2}{2} \la^2 r_\mca(t)\Big)
            \nn
    &+ O(\eta^3) \, .
\end{align}
Substituting $\bx_\mca=\tbx_{0 \mca}$, and using 
\begin{align}
    & \tbx_{0 \mca} = \f{1}{2} (\mbone_\mca \odot H (\bth) \nabla f (\bth) + \la \nabla_\mca \nabla_\mcac^\top f (\bth) \bthac + \la^2 \btha ) \,\,\,\,\, (\text{Equation \ref{tmpeq: 146}}) \\
    & \btha \cdot \tbx_{0 \mca} = \f{1}{2} (- ||\nabla_\mca f (\bth) ||^2 + \la^2 r^2_\mca) \,\,\,\,\, (\text{Equation \ref{tmpeq: 141}}) \, ,
\end{align}
we have
\begin{align}
    \cos \De(t) 
    =& \f{\bthat}{r_\mca (t)} \cdot \f{
            \bthat - \eta \nabla_\mca f (\btht) - \eta \la \bthat
        }{
            r_\mca (t) - \eta \la r_\mca (t) - \f{\eta^2}{2 r_\mca (t)} (- || \nabla_\mca f (\bth) ||^2 + \la^2 r_\mca^2(t) ) + \f{\eta^2}{2}\la^2 r_\mca (t)
        } + O(\eta^3) \\
    =& \f{ (1-\eta \la) r^2_\mca (t) }{(1-\eta \la) r_\mca^2(t) + \f{\eta^2}{2 r_\mca^2 (t)}||\nabla_\mca f (\hbtha + \bthac)||^2 } + O (\eta^3)
    \, .
\end{align}
At equilibrium, we have $r_\mca^2 \xrightarrow{t\rightarrow \infty} r^2_{\mca *} = \sqrt{\f{\eta}{2\la+\eta\la^2}}c_*$ and $|| \nabla_\mca f (\hbthat + \bthac(t)) || \xrightarrow{t\rightarrow \infty} c_*$ because of Corollary \ref{cor: r at equil.}. Thus,
\begin{align}
    \cos \De_* 
    =& \f{(1-\eta\la)r_{\mca *}^2}{(1-\eta\la) r_{\mca *}^2 + \f{\eta^2}{2 r_{\mca *}^2}c_*^2}
    + O(\eta^3) \\
    =& \f{1-\eta\la}{1-\eta^2\la^2/2} + O(\eta^3) \, , \label{tmpeq: cosdestar}
\end{align}
and we have shown the first statement of the theorem. 

The second statement follows from Equation (\ref{tmpeq: cosdestar}). By definition of cosine and tangent, we have
\begin{align}
    \tan \De_* 
    = \f{\sqrt{(1-\eta^2\la^2/2)^2 - (1-\eta\la)^2}}{1-\eta\la} +O(\eta^3) 
    = \f{\sqrt{2\eta\la-2\eta^2\la^2+\eta^4\la^4/4}}{1-\eta} +O(\eta^3) \, .
\end{align}
Therefore, using Taylor's series of the tangent function, we have
\begin{align}
    \De_* = \tan \De_* -\f{1}{3}\De_*^3 - \f{2}{15}\De_*^5 - ... = \sqrt{2\eta\la} + O((\eta\la)^{3/2}) \, .
\end{align}
This concludes the proof.
\end{proof}

\subsection{Proof of Theorem \ref{thm: EoM of bthperp (translation)}} \label{app: Proof of thm: EoM of bthperp (translation)}
We use the following Lemma:
\begin{lemma}\label{lem: gradient constraint and relation of trln-inv. layers}
For translation-invariant layers $\mca$, the following equations hold:
\begin{align}
    & \bthaperp\cdot\bthapara = 0 \\
    & \mbone_\mca \cdot \nabla f (\bth) = \mbone_\mca \cdot \nabla_\mca f (\bth) = 0 \\
    & P \nabla f (\bth) = P \nabla_\mca f (\bth) = 0 \\
    & \bthaperp \cdot \nabla f (\bth) = \bthaperp \cdot \nabla_\mca f (\bth) = 0 \\
    & H(\bth) \mbone_\mca = 0 \\
    & P H(\bth) = 0 \\
    & H(\bth) \bthaperp = 0 \\
    & \nabla f (\bth) = \nabla f (\bthapara + \bthac) \\
    & H(\bth) = H(\bthapara + \bthac) \, .
\end{align}
\end{lemma}
\begin{proof}
Note that $P^\top = P$, $P^2 = P$, and thus, $P^\top (I - P) = P(I-P) = P - P = 0$. Therefore,
\begin{align}
    \bthaperp \cdot \bthapara = \btha^\top P^\top (I-P) \btha = 0 \, .
\end{align}
Next, differentiating $f(\bth) = f(\bth+\al \mbone_\mca)$ with respect to $\al$, we have
\begin{align}
    \mbone_\mca \cdot \nabla f (\bth + \al \mbone_\mca) = 0 \, .
\end{align}
For $\al = 0$, we have 
\begin{align}
    \mbone_\mca \cdot \nabla f (\bth) = \mbone_\mca \cdot \nabla_\mca f (\bth) = 0 \, .
    \label{tmpeq: 1nf}
\end{align}
Therefore, 
\begin{align}
    P \nabla f (\bth) = P \nabla_\mca f (\bth) = (\mbone_\mca \cdot \nabla f (\bth)) \f{1}{d_\mca}\mbone_\mca = 0 
\end{align}
and
\begin{align}
    \bthaperp \cdot \nabla f (\bth) = \f{\mbone_\mca \cdot \btha}{d_\mca} \mbone_\mca \cdot \nabla f (\bth) = \f{\mbone_\mca \cdot \btha}{d_\mca} \mbone_\mca \cdot \nabla_\mca f (\bth) = 0 \, .
\end{align}
Next, differentiating Equation \ref{tmpeq: 1nf} with respect to $\bth$, we have
\begin{align}
    H (\bth) \mbone_\mca = 0 \, .
\end{align}
Therefore,
\begin{align}
    PH(\bth) = \f{\mbone_\mca}{d_\mca} \mbone_\mca^\top H (\bth) = 0
\end{align}
and
\begin{align}
    H(\bth) \bthaperp = \f{\mbone_\mca \cdot \btha}{d_\mca} H(\bth) \mbone_\mca = 0 \, . 
\end{align}
Next, differentiating $f(\bth) = f(\bth + \al \mbone_\mca)$ with respect to $\bth$, we have
\begin{align}
    \nabla f (\bth) = \nabla f(\bth + \al \mbone_\mca)
\end{align}
and
\begin{align}
    H(\bth) = H(\bth + \al \mbone_\mca) \, .
\end{align}
For $\al = - \f{\mbone_\mca \cdot \btha}{d_\mca}$, we have
\begin{align}
    \nabla f (\bth) = \nabla f (\bth - P \btha) = \nabla f (\btha + \bthac - P \btha) = \nabla f (\bthapara + \bthac) 
\end{align}
and 
\begin{align}
    H(\bth) = H(\bthapara + \bthac) \, .
\end{align}
\end{proof}
We begin the proof of Theorem \ref{thm: EoM of bthperp (translation)}.
\begin{proof}
We use Lemma \ref{lem: gradient constraint and relation of trln-inv. layers}.
\begin{align}
    \dbthaperp = P \dot{\bth}_\mca = P(- \nabla_\mca f (\bth) - \la \btha - \eta \bx_\mca) = - \la \bthaperp - \eta P \bx_\mca \, .
\end{align}
When $\bth = \bm{0}$, EoM is
\begin{align}
    \dbthaperpt =  -\la \bthat \, .
\end{align}
When $\bx = \tbx_{0\mca}$, note that 
\begin{align}
    \tbx_{0} = \f{1}{2}(H(\bth) \nabla f (\bth) + \la \nabla f(\bth) + \la H (\bth) \bth + \la^2 \bth)
\end{align}
and
\begin{align}
    \tbx_0 \cdot \mbone_\mca = \tbx_{0\mca} \cdot \mbone_\mca = \f{\la^2}{2}\mbone_\mca \cdot \btha \, .
\end{align}
Thus,
\begin{align}
    P \tbx_0 = \f{\la^2}{2}\bthaperp \, .
\end{align}
Therefore,
\begin{align}
    \dbthaperp = - \la \bthaperp - \eta P \tbx_{0\mca} = -\la \bthaperp -\eta \f{\la^2}{2}\bthaperp = - (\la + \f{\eta \la^2}{2})\bthaperp \, .
\end{align}
Using $\dot{\bm{v}}(t) = -a\bm{v}(t) \Leftrightarrow \bm{v}(t) = \bm{v}(0)e^{-at}$, we can show the remaining equations.
\end{proof}

\subsection{Proof of Theorem \ref{thm: EoM of bthpara (translation)}} \label{app: Proof of thm: EoM of bthpara (translation)}
\begin{proof}
We use Lemma \ref{lem: gradient constraint and relation of trln-inv. layers}.
First, note that 
\begin{align}
    \dbthapara = \dbth_\mca - \dbthaperp \, .
\end{align}
Because 
\begin{align}
    \dbth_\mca = - \nabla_\mca f (\bth) - \la \btha - \eta \bx_\mca
\end{align}
and
\begin{align}
    \dbthaperp = -\la \bthaperp - \eta P \bx_\mca \, ,
\end{align}
we have
\begin{align}
    \dbthapara = \dbth_\mca - \dbthaperp &= - \nabla_\mca f (\bth) - \la \bthapara - \eta (I-P) \bx_\mca \\
    &= - \nabla_\mca f (\bthapara + \bthac) - \la \bthapara - \eta (I-P) \bx_\mca
    \, .
    \label{tmpeq: dbthapara}
\end{align}
Note that $\dbthapara$ is orthogonal to $\bthapara$ because $\bthaperp \cdot \dbthapara = -\bthaperp \cdot \nabla_\mca f (\bth) - \la \bthaperp \cdot \bthapara - \eta \bthaperp^\top (I-P) \tbx_{0\mca} = 0 - 0 - 0 = 0$ (we used $\bthaperp^\top (I-P) = \btha^\top P^\top (I-P) = \btha^\top (P-P) = 0$).

When $\bx = \bm{0}$, we have
\begin{align}
    \dbthapara = - \nabla_\mca f (\bth) - \la \bthapara = -\nabla_\mca f (\bthapara + \bthac) - \la \bthapara \, .
\end{align}
When $\bx = \tbx_0$, we have 
\begin{align}
    \dbthapara 
    &= - \nabla_\mca f (\bth) - \la \bthapara \nn
    &\,\,\,\,\,\,\,- \eta(\f{1}{2}(\mbone_\mca H(\bth) \nabla f (\bth) + \la \nabla_\mca f (\bth) + \la \mbone_\mca \odot H (\bth) \bth + \la^2 \btha) - \f{\la^2}{2}\bthaperp) \\
    &= -\nabla_\mca f (\bth) - \la \bthapara - \eta ( \f{1}{2}\mbone_\mca \odot H (\bth) \nabla f (\bth) + \f{\la}{2} \nabla_\mca f (\bth) + \f{\la}{2}\mbone_\mca \odot H (\bth) \bth  + \la{\la^2}{2} \bthapara ) \\
    &= - \la \bthapara - \f{\eta \la^2}{2} \bthapara - \nabla_\mca f (\bth) -\f{\eta \la}{2} \nabla_\mca f (\bth) - \f{\eta}{2} \mbone_\mca \odot H (\bth) \nabla f (\bth) - \f{\eta \la}{2} \mbone_\mca \odot H (\bth) \bth \\
    &= - (1+\f{\eta\la}{2})(\nabla_\mca f (\bth) + \la \bthapara) \nn
        & \,\,\,\,\,\,\, - \f{\eta}{2}H_\mca(\bth) \nabla_\mca f (\bth) - \f{\eta}{2} \nabla_\mca \nabla_\mcac^\top f (\bth) \nabla_\mcac f (\bth) - \f{\eta\la}{2} H_\mca (\bth) \btha - \f{\eta \la}{2} \nabla_\mca \nabla_\mcac^\top f (\bth) \bthac \\
    &= - (I + \f{\eta\la}{2} I + \f{\eta}{2} H_\mca(\bthapara + \bthac)) (\nabla_\mca f (\bthapara + \bthac) + \la \bthapara) \nn
    & \,\,\,\,\,\,\, - \f{\eta} {2} \nabla_\mca \nabla_\mcac^\top f (\bthapara + \bthac) (\nabla_\mcac f (\bthapara + \bthac) + \la \bthac) \\
    &= -\la (I + \f{\eta \la}{2} I + \f{\eta}{2} H_\mca (\bthapara + \bthac)) \bthapara - (I + \f{\eta\la}{2}I + \f{\eta}{2} H_\mca (\bthapara + \bthac) \nn
    & \,\,\,\,\,\,\,  + \f{\eta}{2}I ((\nabla_\mcac f (\bthapara + \bthac)+\la\bthac) \cdot \nabla_\mcac)) \nabla_\mca f (\bthapara + \bthac) \, .
\end{align}
\end{proof}

\subsection{Proof of Theorem \ref{thm: Equation of Learning}} \label{app: Proof of thm: Equation of Learning}
\begin{proof}
First, note that
\begin{align}
    \nabla f (\bth) \cdot \bG(\bth, \al) = 0 \, ,
\end{align}
which can be shown by differentiating $f(\bth) = f(\bG(\bth, \al))$ with respect to $\al$.
Thus, assuming $\bth \cdot ((\nabla f (\bth) \cdot \nabla) \bG(\bth, \al))$ and using $\dbtht = -  \nabla f (\btht) - \la \btht - \eta \bx(\btht)$, we have
\begin{align}
    &\f{d}{dt}(\btht \cdot \bG(\btht, \al)) \\
    =& \dbth \cdot \bG (\bth, \al) + \bth \cdot (\dbth \cdot \nabla \bG(\bth, \al)) \\
    =& - \nabla f (\bth) \cdot \bG (\bth, \al) - \la \bth \cdot \bG (\bth, \al) - \eta \bx(\bth) \cdot \bG (\bth, \al) + \bth \cdot (-(\nabla f (\bth) \cdot \nabla) - \la (\bth \cdot \nabla) \nn
    &- \eta \bx(\bth) \cdot \nabla ) \bG (\bth, \al) \\
    =& - \la (\bth \cdot \bG(\bth,\al) + \bth \cdot ((\bth \cdot \nabla) \bG(\bth, \al))) - \eta \bx(\bth) \cdot \bG(\bth, \al) - \bth \cdot ((\eta \bx(\bth) \cdot \nabla) \bG(\bth, \al)) \, .
\end{align}
Using $\dot{\bm{v}}(t) = - a \bm{v}(t) + \bm{u}t \Leftrightarrow \bm{v}(t) = \bm{v}(0) e^{-at} + \int_0^t d\tau e^{-a(t-\tau)} \bm{u}(\tau)$, we have
\begin{align}
    &\,\,\btht \cdot \bG(\btht, \al) \\
    =&\,\, \bth(0) \cdot \bG(\bth(0), \al)  \\
    &- \la \int_0^t d\tau e^{-\la(t-\ta)} \bth(\tau) \cdot ((\bth(\ta)\cdot \nabla) \bG(\bth(\ta), \al)) \\
    &- \eta \int_0^t d \tau e^{-\la (t-\ta)} (\bx(\bth(\ta)) \cdot \bG(\bth(\ta), \al) + \bth(\ta) \cdot ((\bx(\bth(\ta))\cdot \nabla ) \bG(\bth(\ta), \al)) )  \, .
\end{align}

\end{proof}

%% file: app_broken_symmetry.tex
\clearpage
\section{Learning Dynamics Induced by Symmetry Breaking: Neural Mechanics} 
\label{app: Broken Conservation Laws}
To show the benefits of the counter term, we apply it to broken conservation laws \cite{kunin2021neural_mechanics_1}.
In \cite{kunin2021neural_mechanics_1}, the authors build relationships between the symmetries of weights and conserved quantities (i.e., Noether's theorem \cite{noether1918invariante_variationsprobleme, noether1971invariant_variation_problem} for DNNs), and they also investigate the dynamics of DNNs under symmetry breaking. 
We address three shortcomings of their analysis: 1) it includes a counter term only up to order one, 2) a discretization error analysis is missing, and 3) their experiment makes too optimistic an assumption on gradients.

First, we generalize broken conservation laws (Equations (18--20) in \cite{kunin2021neural_mechanics_1}) by adding all orders of the counter term. 
Let $\bG (\bth, \al) := \partial_\al \bm{\psi}(\bth, \al)$, which is called the generator of symmetry transformation $\bpsi$.
\begin{theorem}[Generalized broken conservation law] \label{thm: Equation of Learning}
Let $f$ be symmetric under transformation $\bpsi$. Assume that $\bm{G}$ satisfies $\btht \cdot \{ (\nabla f (\btht) \cdot \nabla) \bgbtht\} = 0$. 
Then,
\begin{align}
    &\f{d}{dt} ( \btht \cdot \bm{G} (\btht, \al)) = \nn
    &- \la \btht \cdot \bgbtht - \la \btht \cdot \{ (\btht \cdot \nabla) \bgbtht \}
    - \eta (\bm{\xi} (\btht) \cdot \nabla) \cdot (\bm{\xi} (\btht) \cdot \bgbtht) \, .
    \label{eq: EoL}
\end{align}
Note that the assumption holds for translation, scale, and rescale transformation \cite{kunin2021neural_mechanics_1}.
Furthermore, Equation (\ref{eq: EoL}) can be formally solved:
\begin{align}
    &\btht \cdot \bm{G} (\btht, \al) = 
    \bth (0) \cdot \bm{G} (\bth (0), \al) e^{-\la t}\nn
    &- \la \int_0^t e^{-\la(t - \tau)} \bth (\tau) \cdot \{ (\bth (\tau) \cdot \nabla) \bm{G} (\bth (\tau), \al) \} d\tau \nn
    &- \eta \int_0^t e^{-\la(t - \tau)} (\bm{\xi}(\bth(\tau)) \cdot \nabla) (\bm{\xi} (\bth (\tau)) \cdot \bm{G} (\bth(\tau), \al)) d\tau \, .
    \label{eq: integral of EoL}
\end{align}
\end{theorem}
The proof is given in Appendix \ref{app: Proof of thm: Equation of Learning}.
Now, Equation (\ref{eq: integral of EoL}) includes all orders of the counter term $\bx = \sum_{\al=0}^\infty \tbx_\al$. 
We can reproduce \cite{kunin2021neural_mechanics_1} by setting $\bx = \tbx_0$.
In addition, we already know the discretization error (Corollary \ref{cor: Discretization error of bx}), which is lacking in \cite{kunin2021neural_mechanics_1}.
We also provide empirical results on Equation (\ref{eq: integral of EoL}) in the following sections.

\subsection{Scale-invariant Layers} \label{app: Scale-invariant Layers}
For scale transformation, $\bG (\bth, \al) = \al_\mca \bth$, and thus, the left hand side of Equation (\ref{eq: integral of EoL}) becomes $||\btha||^2$. Therefore, Equation \ref{eq: integral of EoL} describes the temporal evolution of the weight norm of scale-invariant layers. Figure \ref{fig: norm squared} shows the temporal evolution of $||\btha||^2$ for the network explained in Section \ref{sec: Experiment}. Figure \ref{fig: norm squared gap} shows the gap of $||\btha||^2$ between GD and its theoretical predictions (GF and EoM) (Equation \ref{eq: integral of EoL}). We see that the counter term reduces the gap.
There is an improvement in the experimental settings compared with \cite{kunin2021neural_mechanics_1}. As described in \cite{kunin2021neural_mechanics_1}, they substitute the gradients computed in GD for the gradients used for GF's simulation instead of using small learning rates to simulate continuous trajectories of GF. This approximation reduces computational costs, but it causes an additional gap between the surrogate gradients and the true gradients of GF along the continuous trajectories. Therefore, we avoid this approximation; we use a small learning rate ($\eta=10^{-5}$) to simulate GF and EoM, as explained in Section \ref{sec: Experiment}.

\begin{figure}[htbp]
	\centering
	\includegraphics[width=0.8\linewidth]
      {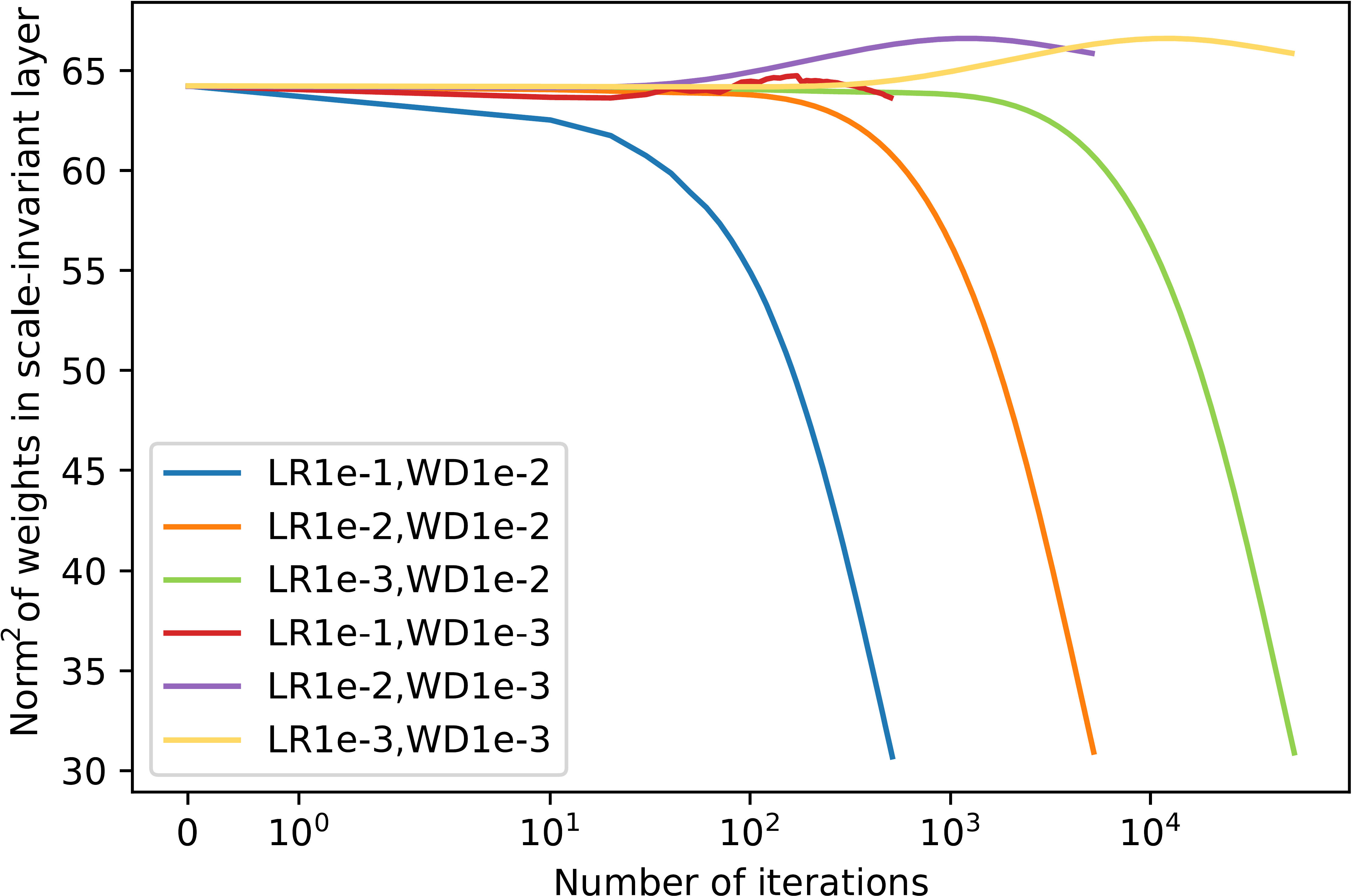}
	\caption{\textbf{Dynamics of squared weight norm of scale-invariant layer.} LR and WD mean learning rate and weight decay, respectively. See Section \ref{sec: Experiment} for experimental settings.}
	\label{fig: norm squared}
\end{figure}

\begin{figure}[htbp]
	\centering
	\includegraphics[width=0.8\linewidth]
      {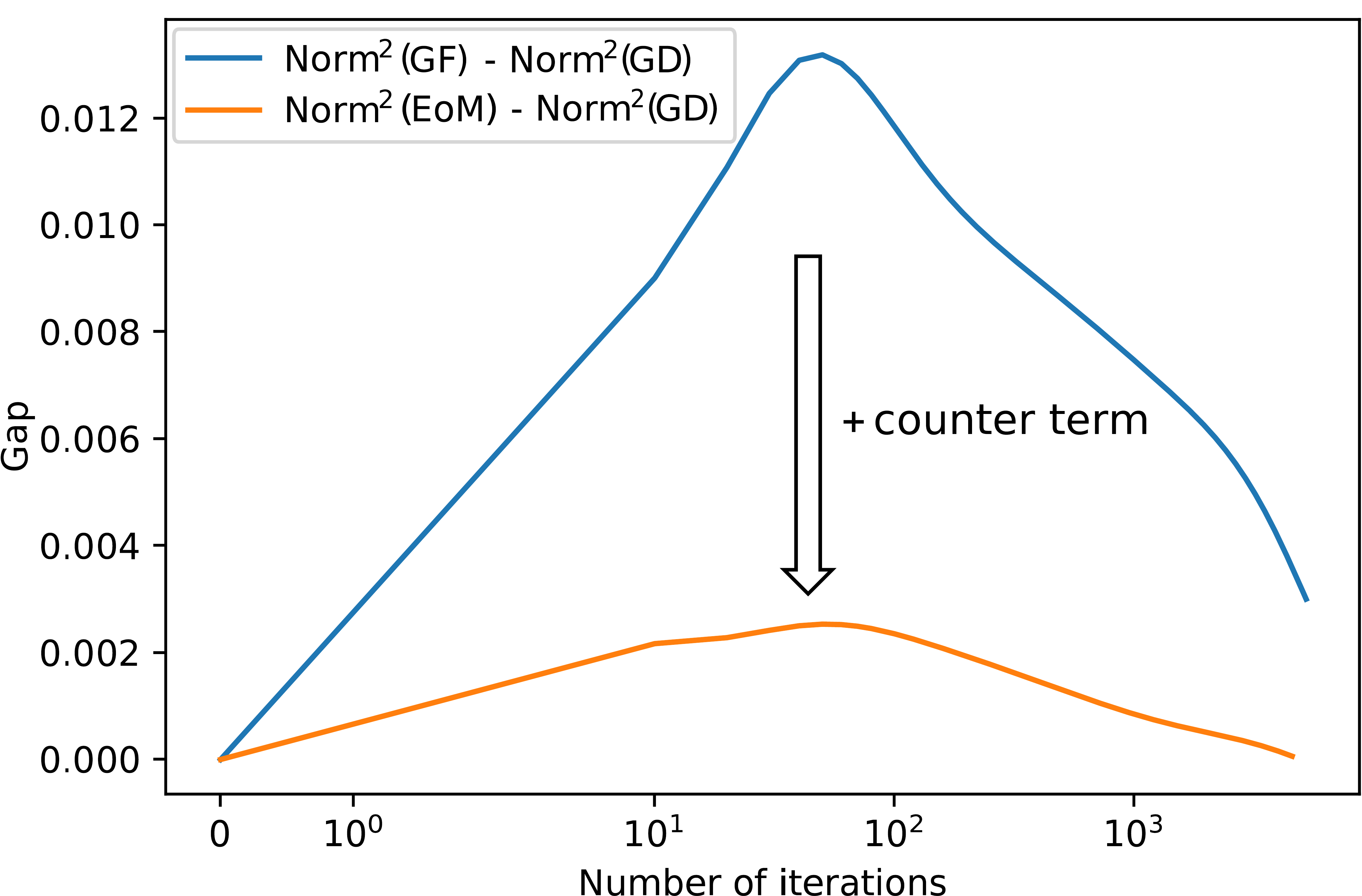}
	\caption{\textbf{Discrepancy between actual dynamics of GD and its theoretical prediction (GF and EoM) of squared weight norm of scale-invariant layer.} We see that our counter term reduces the gap between the actual dynamics of GD and its theoretical prediction. See Section \ref{sec: Experiment} for experimental settings.}
	\label{fig: norm squared gap}
\end{figure}

\clearpage
\subsection{Translation-invariant Layers}
We also provide an empirical result for translation-invariant layers. For translation transformation, $\bG (\bth, \al) = \al \mbone_\mca$ and thus the left hand side of Equation (\ref{eq: integral of EoL}) becomes $\mbone_\mca \cdot \btha$ (sum of weights). Therefore, Equation (\ref{eq: integral of EoL}) describes the temporal evolution of the sum of weights of translation-invariant layers. Figure \ref{fig: sum of weights} shows the temporal evolution of $\mbone_\mca \cdot \btha$ for the network described in Section \ref{sec: Experiment}. Figure \ref{fig: sum of weights gap} shows the gap of $\mbone_\mca \cdot \btha$ between GD and its theoretical predictions (GF and EoM) (Equation \ref{eq: integral of EoL}). We see that the counter term reduces the gap.

\begin{figure}[htbp]
	\centering
	\includegraphics[width=0.8\linewidth]
      {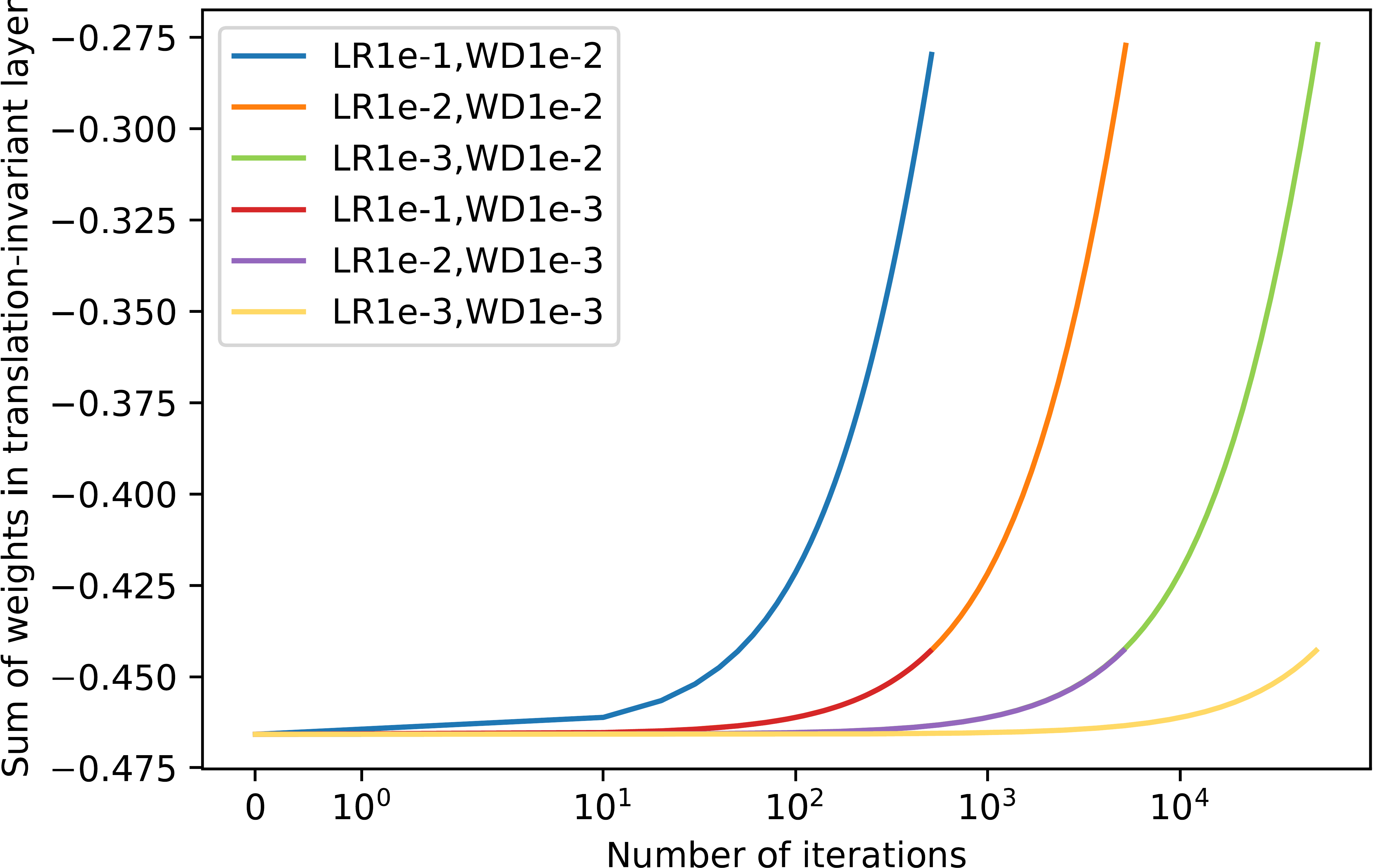}
	\caption{\textbf{Sum of weights of translation-invariant layer.} LR and WD mean learning rate and weight decay, respectively. See Section \ref{sec: Experiment} for experimental settings.}
	\label{fig: sum of weights}
\end{figure}

\begin{figure}[htbp]
	\centering
	\includegraphics[width=0.8\linewidth]
      {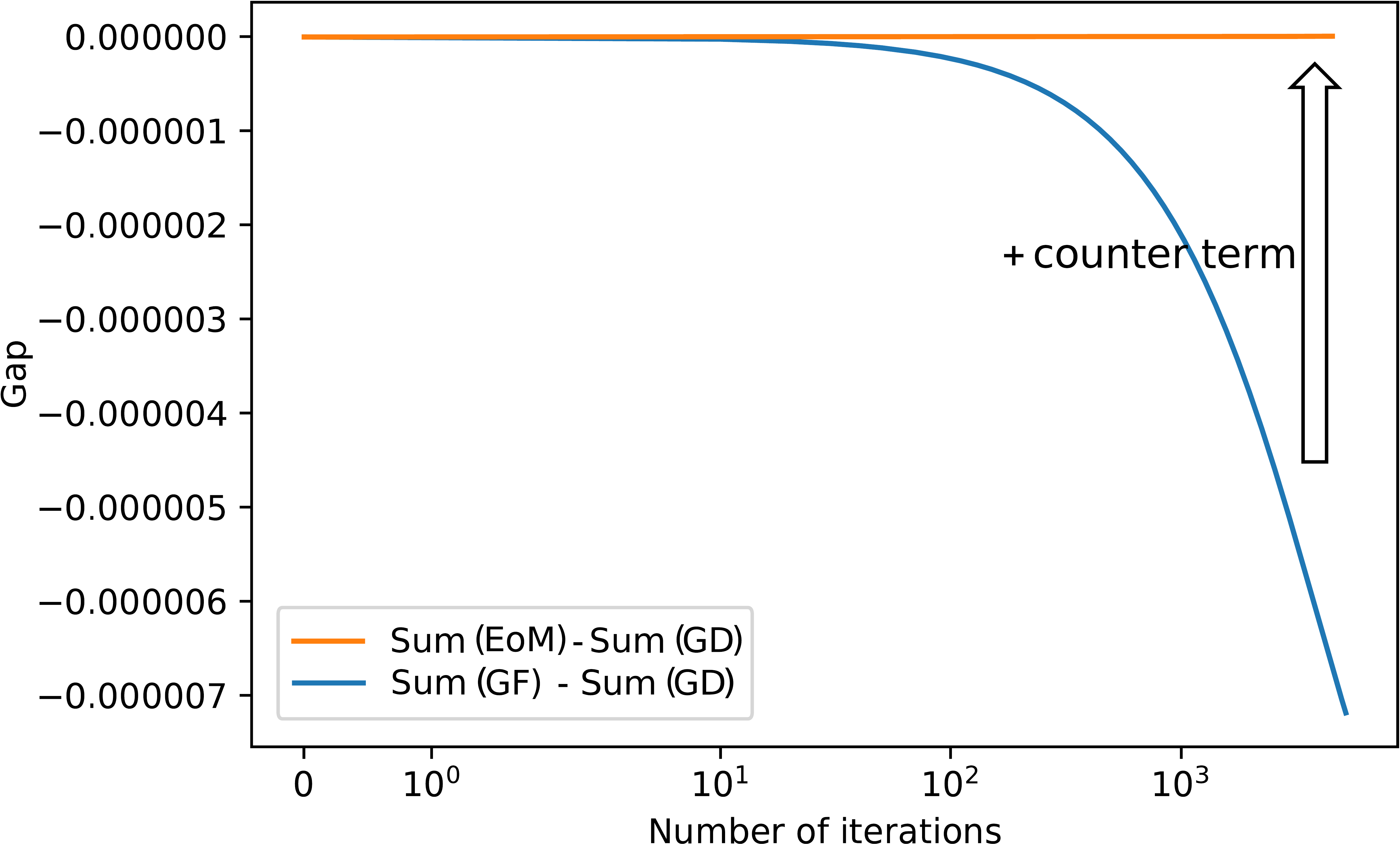}
	\caption{\textbf{Discrepancy between actual dynamics of GD and its theoretical prediction (GF and EoM) of sum of weights of translation-invariant layer.} We see that our counter term reduces the gap. See Section \ref{sec: Experiment} for experimental settings.}
	\label{fig: sum of weights gap}
\end{figure}

%% file: app_effectiveLR.tex
\clearpage
\section{Equation of Motion for 
$\hbth_\mca$} 
\label{app: Equation of Motion for hbth}
For completeness, we construct the EoM for $\hbth_\mca$ for scale-invariant layers $\mca$. See Section \ref{sec: Learning Dynamics of Scale-invariant Layers} for the EoM for $r_\mca$.
\begin{theorem}[EoM for $\hbtha$] \label{thm: EoM of hattheta}
EoM (\ref{eq: gradient flow}) gives $\dot{\hbth}_\mca(t) = - \f{1}{r_\mca^2(t)} \nabla_\mca f(\hbth_\mca(t)) + \f{\eta}{r_\mca(t)} ( (\hbth_\mca(t) \cdot \bx(\btht)) \, \hbth_\mca(t) - \bx(\btht) )$.
Specifically, this is equivalent to:
\begin{align}
    \dot{\hbth}_\mca (t) = - \f{1}{r_\mca^2(t)} \nabla_\mca f(\hbth_\mca(t))
    \label{eq: EoM of dhbtht with xi=0}
\end{align}
for $\bx = \bm{0}$ (GF) and
\begin{align}
    \dhbth_\mca 
    = - \f{1}{r_\mca^2} \left(
            I 
            + \f{\eta}{2} H_\mca (\bth) 
            + \f{\eta}{2} I ((\nabla_\mcac f (\bth) + \la \bthac) \cdot \nabla_\mcac) 
            + \f{\eta}{2} \hbtha \nabla_\mca^\top f(\bth)
        \right) \nabla_\mca f (\hbtha + \bthac) \, 
    \label{eq: EoM of dhbtht with xi=xi0}
\end{align}
for $\bx = \tbx_0$ (EoM), where $H_\mca (\hbth_\mca) := (\mbone_\mca \odot \nabla) (\mbone_\mca \odot \nabla)^\top f (\bth) |_{\bth = \hbth_\mca}$.
\end{theorem}
The proof is given in Appendix \ref{app: Proof of thm: EoM of hattheta}.

\paragraph{Effective learning rate.}
This result highlights the differences between GD and GF on scale-invariant layers.
The factor $\f{1}{r_\mca^2}$ (Equation (\ref{eq: EoM of dhbtht with xi=0})), which is $\f{\eta}{r_\mca^2}$ at discretization, is called the \textit{effective learning rate} \cite{van2017l2_effectiveLR_orginal1, hoffer2018norm, zhang2018three, arora2018theoretical, chiley2019online, zhiyuan2020reconciling_NIPS2020, li2020exponential_learning_rate_ICLR2020, wan2021spherical_motion_dynamics, li2022robust, roburin2022spherical}. The dynamics of $\hbtha$ is induced by $\nabla_\mca f (\hbth_\mca + \bthac)$ with the effective learning rate $\f{\eta}{r_\mca^2}$, not $\eta$.
We find that the counter term corrects the effective learning rate to a matrix operator form (Equation (\ref{eq: EoM of dhbtht with xi=xi0})).
Let us see the meaning of each correction in order.
First, $I$ (identity matrix) corresponds to the original effective learning rate.
Second, $\f{\eta}{2} H_\mca$ directs the gradient $\nabla_\mca f(\hbtha + \bthac)$ toward the maximum eigenvector of $H_\mca$, i.e., a flat direction. Therefore, GD tends to go through flatter regions than GF.
Third, $\f{\eta}{2} I ((\nabla_\mcac f (\bth) + \la \bthac) \cdot \nabla_\mcac)$ involves $\nabla_{\mcac} f$ into the learning dynamics of $\mca$; therefore, $\mca$ is explicitly affected by $\mcac$ in GD, unlike in GF. This point is often missing in the literature on scale-invariant networks because it is often assumed that the whole network is scale-invariant. 
Fourth, $\f{\eta}{2} \hbtha \nabla_\mca^\top f(\bth)$ cancels the $\hbtha$ component of the right hand side of Equation (\ref{eq: EoM of dhbtht with xi=xi0}), which may not seem obvious but can be seen from the proof of Theorem \ref{thm: EoM of hattheta} (see Appendix \ref{app: Proof of thm: EoM of hattheta}), and thus, $\dot{\hbth}_\mca$ is orthogonal to $\hbtha$, which should be satisfied anyway because $|| \hbtha ||^2 \equiv 1 \Longrightarrow 2 \dot{\hbth}_\mca \cdot \hbtha = 0$.


%% file: app_bthapara.tex
\clearpage
\section{Equation of Motion for ${\theta}_{\mca\parallel}$} \label{app: Equation of Motion for bthapara}
For completeness, we provide the EoM for $\bthpara$. The proof is given in Appendix \ref{app: Proof of thm: EoM of bthpara (translation)}.
\begin{theorem}[EoM for $\bthpara$] \label{thm: EoM of bthpara (translation)}
EoM (\ref{eq: gradient flow}) gives
\begin{align}
    \dbthparat = -\la \bthparat - \nabla f(\bthparat) - \eta (I-P)\bx(\btht) \, .
\end{align}
Specifically, this is equivalent to:
\begin{align}
    &\,\,\dbthparat = -\la \bthparat - \nabla f(\bthparat + \bthac) 
    \label{eq: EoM of bthapara with xi=0}
\end{align}
for $\bx = \bm{0}$ (GF) and
\begin{align}
    &\dbthparat 
    =-\la (I + \f{\eta \la}{2} I + \f{\eta}{2} H_\mca (\bthapara + \bthac)) \bthapara \nn
    &  - \Big(I + \f{\eta\la}{2}I + \f{\eta}{2} H_\mca (\bthapara + \bthac) + \f{\eta}{2}I ((\nabla_\mcac f (\bthapara + \bthac)+\la\bthac) \cdot \nabla_\mcac)\Big) \nabla_\mca f (\bthapara + \bthac)
    \label{eq: EoM of bthapara with txi0}
\end{align}
for $\bx = \tbx_0$ (EoM).
\end{theorem}
This result highlights the differences between the dynamics of GD and GF.
The two factors $\f{\eta\la}{2}I$ in Equation (\ref{eq: EoM of bthapara with txi0}) mean that the existence of weight decay increases the learning rate (increases the velocity $\dbthpara$). 
The factor $\f{\eta}{2}H$ means that, as mentioned in Appendix \ref{app: Equation of Motion for hbth}, GD tends to go along sharper paths than GF.
Note that velocity $\dbth_{\mca\parallel}$ is orthogonal to $\bthperp$ because $\nabla f$, $\bthpara$, and $H (\nabla f + \la \bthpara)$ are orthogonal to $\bthperp$.
$H (\nabla f + \la \bthpara) \perp \bthperp$ follows because $H\bm{v} \perp \bthperp$ for arbitrary non-zero vector $\bm{v} \in \mbr^d$ ($\because H \mbone_\mca = H\bthperp=0$) (see Lemma \ref{lem: gradient constraint and relation of trln-inv. layers}).
$\f{\eta}{2}I ((\nabla_\mcac f (\bthapara + \bthac)+\la\bthac) \cdot \nabla_\mcac)$ involves $\nabla_{\mcac} f$ into the learning dynamics of $\mca$.
We see that the dynamics of $\bthpara$ is also independent of that of $\bthperp$, and thus, they are completely separable.
A summary of Theorems \ref{thm: EoM of bthperp (translation)} and \ref{thm: EoM of bthpara (translation)} is given in Figure \ref{fig: dynamics of trln. inv. layers}.

%% file: app_details_of_experiment.tex
\clearpage
\section{Details of Experiment} \label{app: Details of Experiment}
We provide detailed experimental settings (see also Section \ref{sec: Experiment}).
Our computational infrastructure is a DGX-1 server.
The fundamental libraries used in the experiment are TensorFlow 2.3 \cite{tensorflow}, Numpy 1.18 \cite{numpy}, and Python 3.6.8 \cite{Python3}.
The random seeds used for TensorFlow and Numpy are both 7.
The input image is first divided by 127.5 and subtracted by 1.
The maximum total number of iterations is 5 million steps for GF and EoM. 
The total runtime is approximately a month.
We use least square fitting (\texttt{np.polyfit}) to calculate the decay rates in Table \ref{tab: decay rates of bthaperp}.
More information and detailed experimental results can be found in our code.

In Figures \ref{fig: th vs ex discerr} and \ref{fig: th. pred. vs. exp. of disc. err.}, the theoretical prediction of discretization error is defined as $|| \mathbf{e}_{k} || = \frac{\eta^2}{2} || \sum_{s=0}^{k-1} ( H (\mathbf{\theta}(s\eta)) + \lambda I) \mathbf{g} (\mathbf{\theta}(s\eta)) ||$ (Equation (\ref{eq: disc. err. with no xi})).
To reduce computational costs, we approximate the r.h.s.: $(H({\bf \theta}(t)) + \lambda I){\bf g} (\btht) \sim \frac{{\bf g} ({\bf \theta}(t) + \epsilon {\bf g} ({\bf \theta}(t))) - {\bf g} ({\bf \theta}(t) - \epsilon {\bf g} ({\bf \theta}(t)))}{2\epsilon}$, where $\ep$ is set to $10^{-7}$.
The green curve in Figure \ref{fig: th vs ex discerr} is defined as $\berr_k = \ti{\berr}_{100} + \frac{\eta^2}{2} \sum_{s=100}^{k-1}  ( H (\mathbf{\theta}(s\eta)) + \lambda I) \mathbf{g} (\mathbf{\theta}(s\eta))$ (compare this with Equation (\ref{eq: disc. err. with no xi})), where $\ti{\berr}_{100}$ is the actual discretization error at the 100th step that is obtained from GD. Therefore, the green curve represents the theoretical prediction of discretization error after the 100th step, given $\ti{\berr}_{100}$.


%% file: app_suppl_experiment.tex
\clearpage
\section{Supplementary Experiment} \label{app: Supplementary Experiment}
\subsection{Relative Discretization Error} \label{app: Discretization Error}
We provide the relative discretization error, which is defined as $||\berr_k||/||\bth_k||$ ($k \in \mbz_{\geq 0}$).  See Figure \ref{fig: relaive disc. err.}. We can see that a large learning rate ($\eta = 10^{-1}$) leads to a large discretization error (Figure \ref{fig: relaive disc. err.} (a) and (c)).
We also see that the counter term reduces the discretization error as expected (Figure \ref{fig: relaive disc. err.} (b) and (d)). 

\begin{figure}[htbp]
    \begin{minipage}[b]{0.5\linewidth}
        \centering
        \includegraphics[width=0.9\linewidth]
        {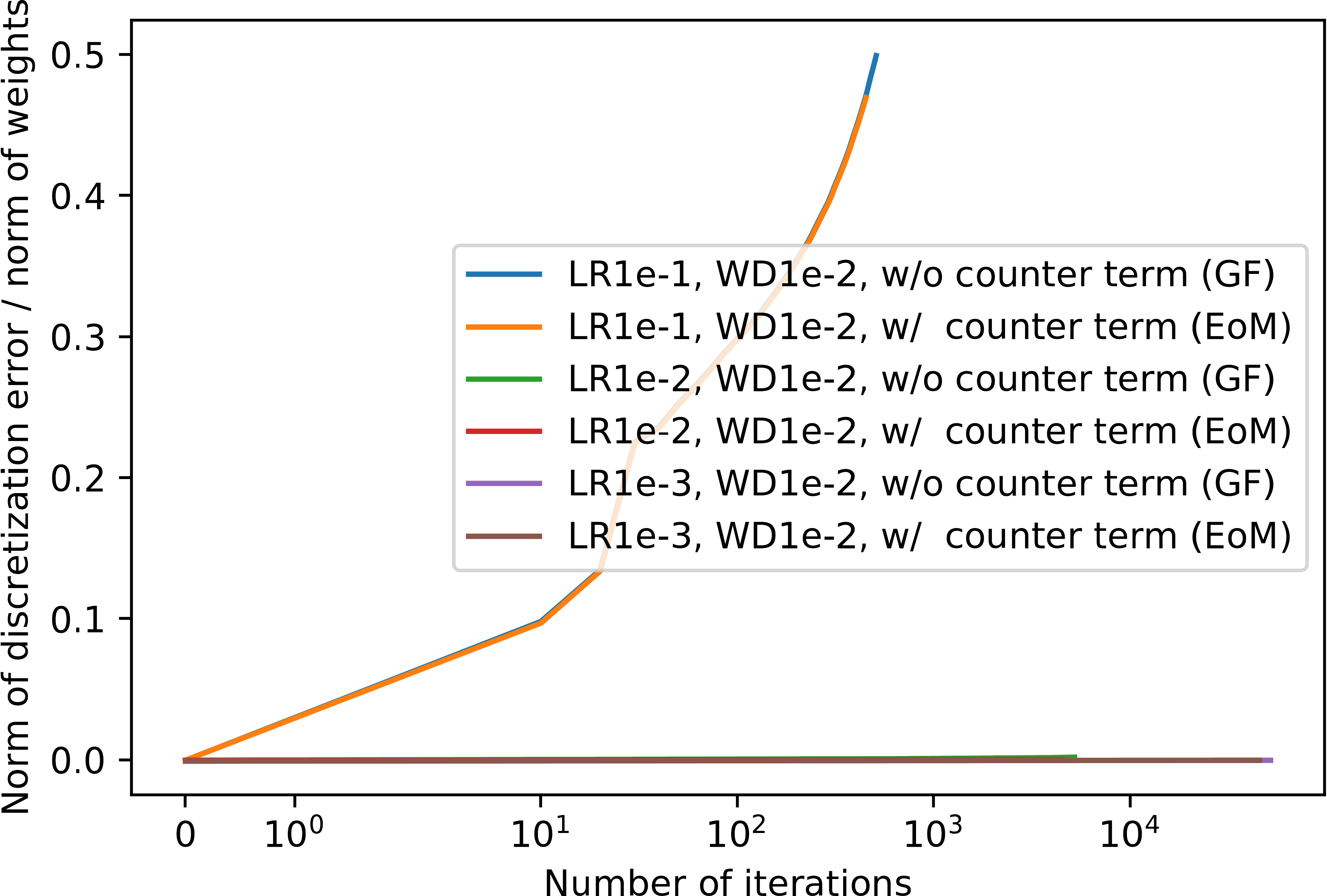}
        \subcaption{Weight decay = $10^{-2}$.}
    \end{minipage}
    \begin{minipage}[b]{0.5\linewidth}
        \centering
        \includegraphics[width=0.9\linewidth]
        {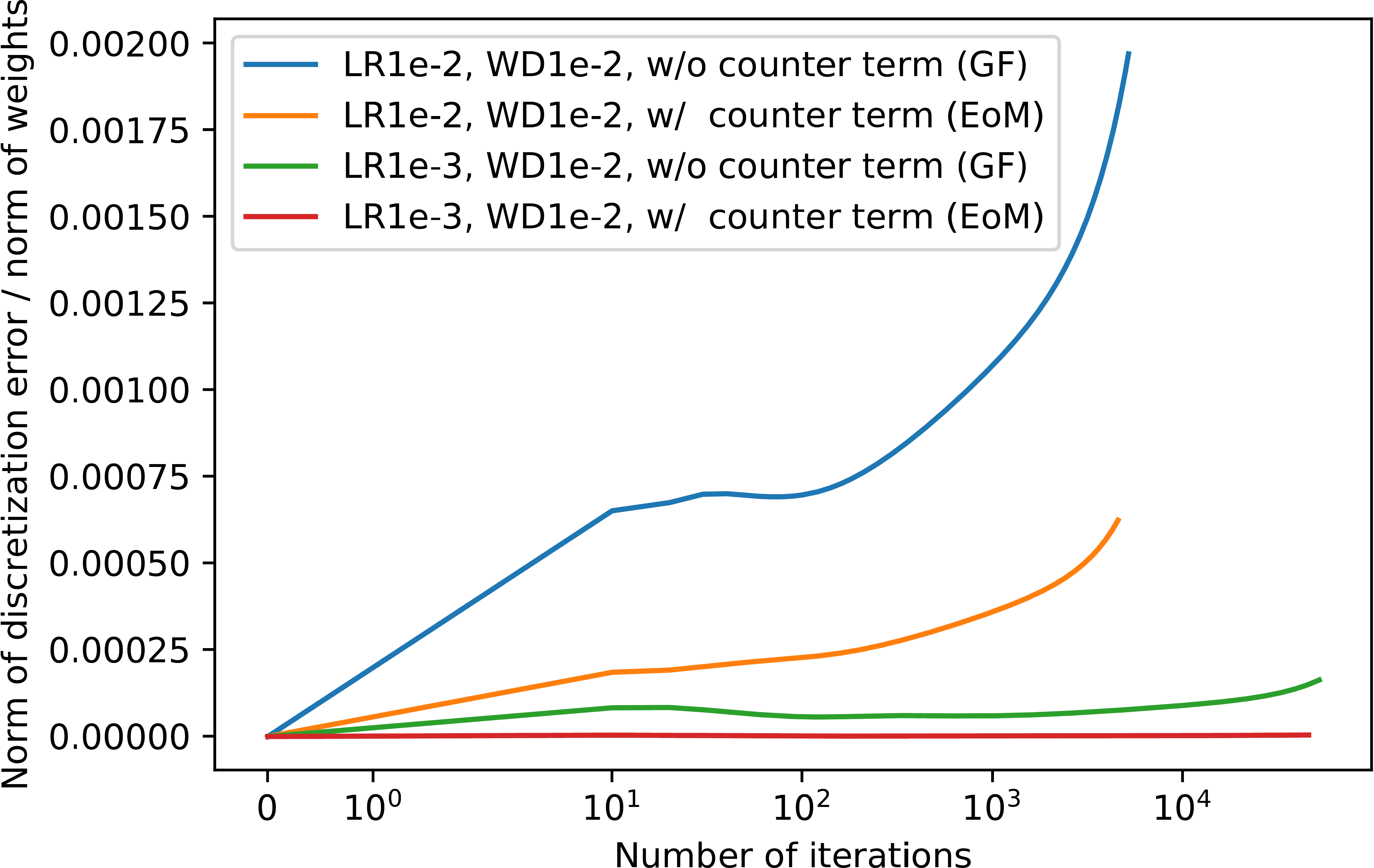}
        \subcaption{Weight decay = $10^{-2}$. Magnified.}
    \end{minipage}\\
    \begin{minipage}[b]{0.5\linewidth}
        \centering
        \includegraphics[width=0.9\linewidth]
        {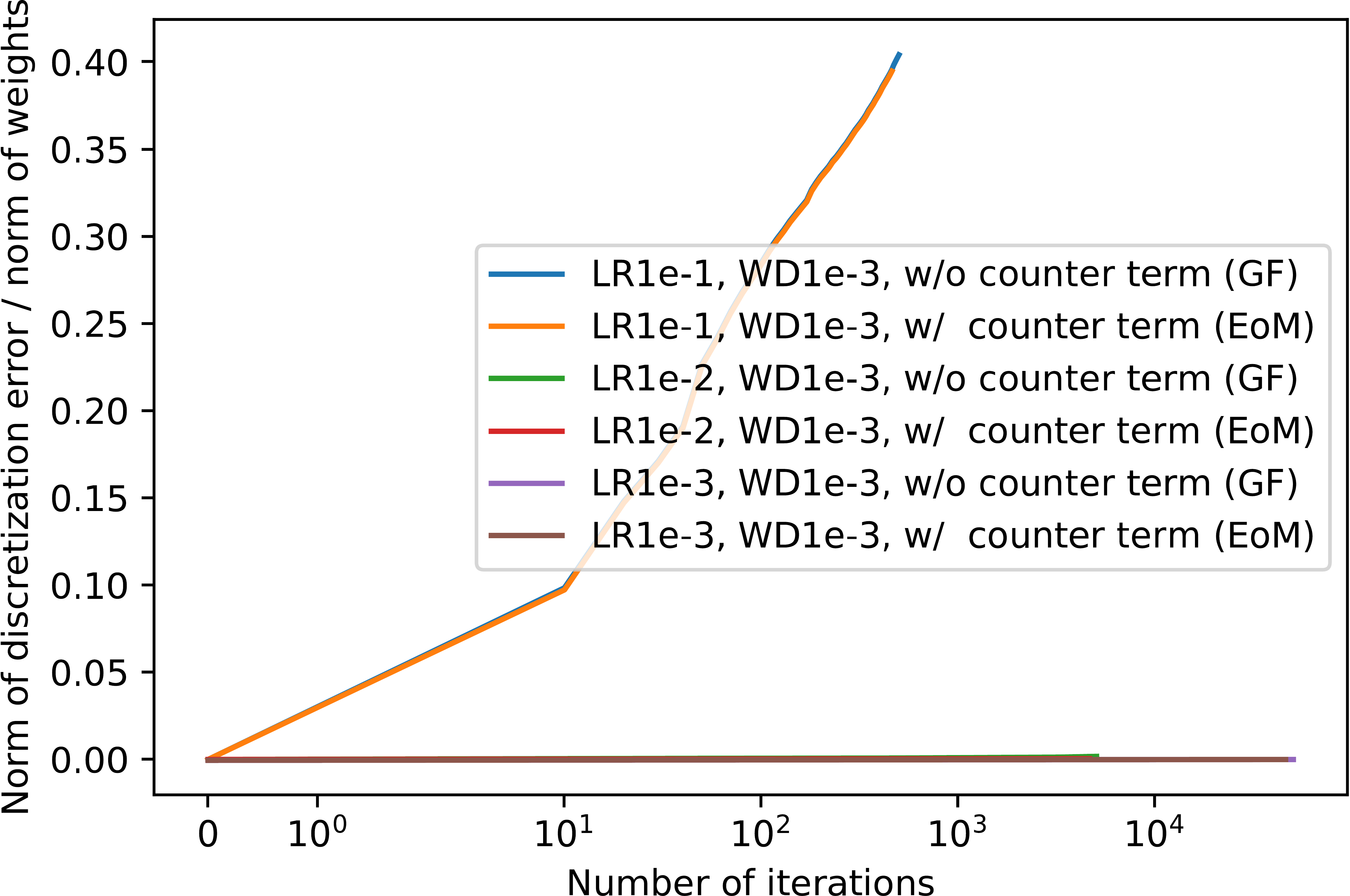}
        \subcaption{Weight decay = $10^{-3}$.}
    \end{minipage}
    \begin{minipage}[b]{0.5\linewidth}
        \centering
        \includegraphics[width=0.9\linewidth]
        {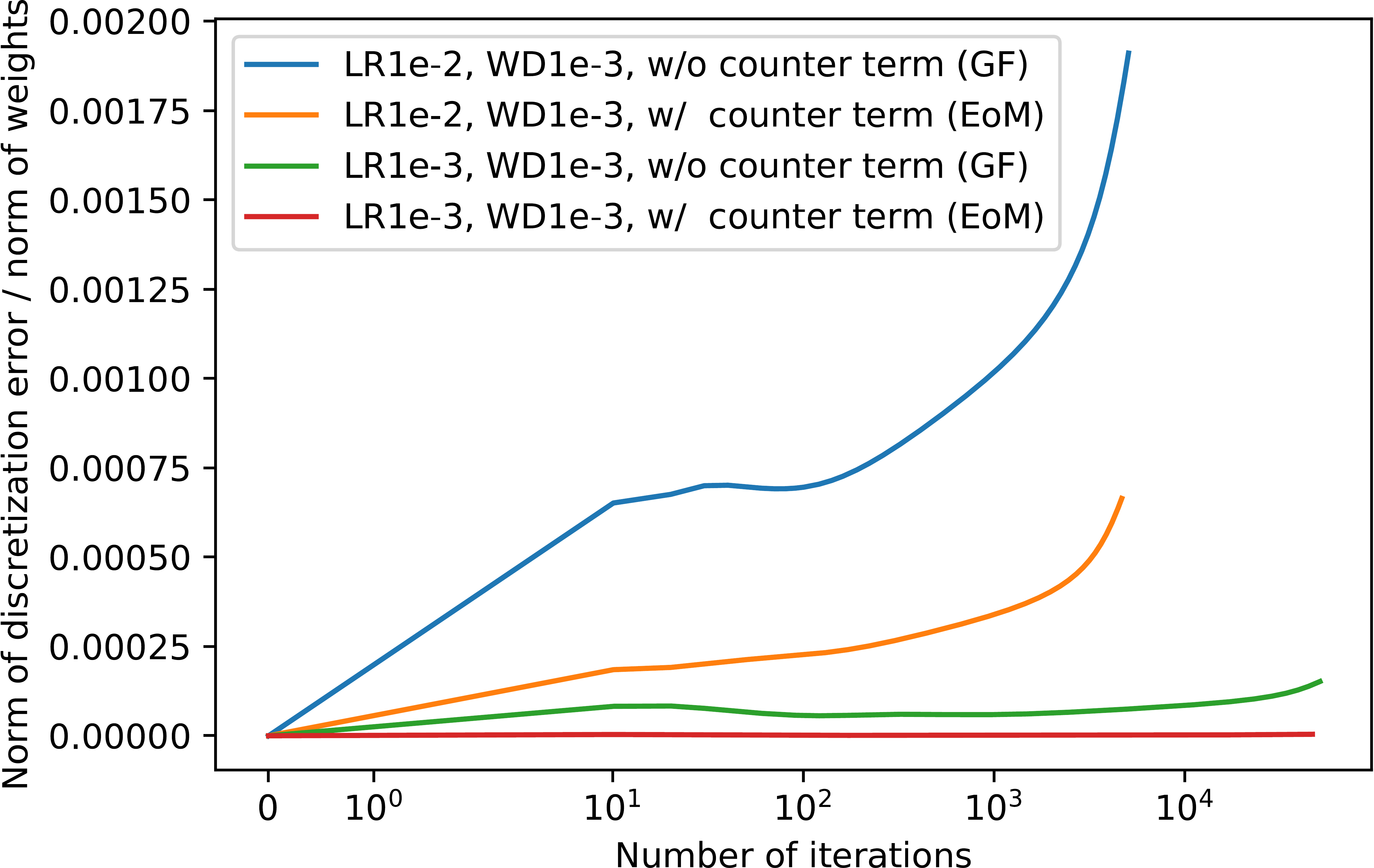}
        \subcaption{Weight decay = $10^{-3}$. Magnified.}
    \end{minipage}
    \caption{\textbf{Relative discretization error.} In (a) and (c), the LR1e-1 curves overlap each other, and the LR1e-2 and LR1e-3 curves collapse in the lower region of the figure. The LR1e-2 and LR1e-3 are magnified and shown in (c) and (d). See Section \ref{sec: Experiment} and Appendix \ref{app: Details of Experiment} for experimental settings.}
    \label{fig: relaive disc. err.}
\end{figure}

\clearpage
\subsection{Theoretical Prediction Vs. Experimental Result of Discretization Error} \label{app: Theoretical prediction vs. experimental result of discretization error}
We compare the theoretical prediction of discretization error between GF and GD (Equation (\ref{eq: disc. err. with no xi})) with the actual discretization error obtained in the experiment.
The green curve is defined as $\berr_k = \ti{\berr}_{100} + \frac{\eta^2}{2} \sum_{s=100}^{k-1}  ( H (\mathbf{\theta}(s\eta)) + \lambda I) \mathbf{g} (\mathbf{\theta}(s\eta)) + O(\eta^3)$ (compare this with Equation (\ref{eq: disc. err. with no xi})), where $\ti{\berr}_{100}$ is the actual discretization error at the 100th step. Therefore, the green curve represents the theoretical prediction of discretization error after the 100th step given $\ti{\berr}_{100}$.

\begin{figure}[htbp]
    \begin{minipage}[b]{0.5\linewidth}
        \centering
        \includegraphics[width=0.9\linewidth]
        {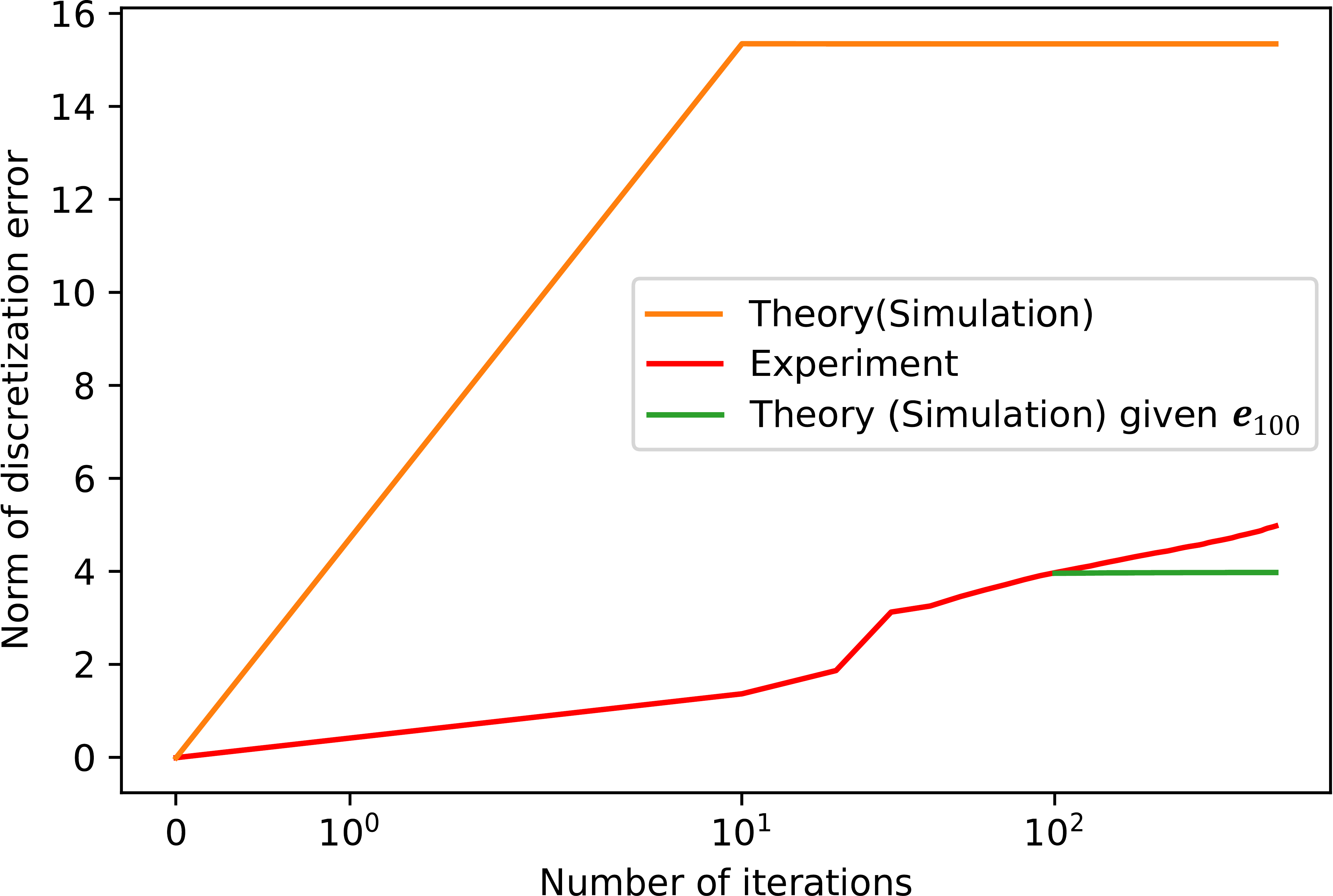}
        \subcaption{Learning rate = $10^{-1}$.}
    \end{minipage}
    \begin{minipage}[b]{0.5\linewidth}
        \centering
        \includegraphics[width=0.9\linewidth]
        {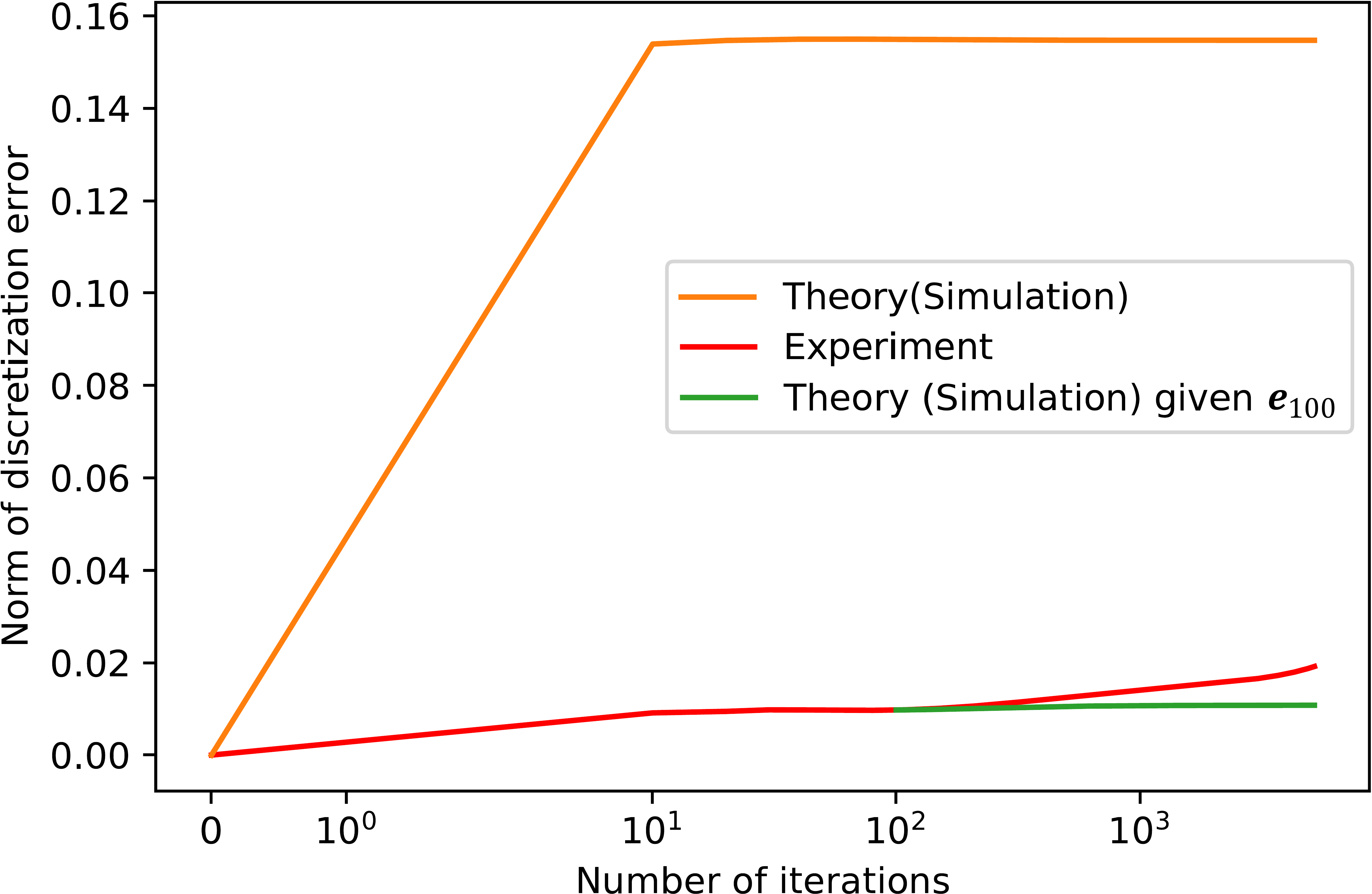}
        \subcaption{Learning rate = $10^{-2}$.}
    \end{minipage}\\
    \begin{minipage}[b]{0.5\linewidth}
        \centering
        \includegraphics[width=0.9\linewidth]
        {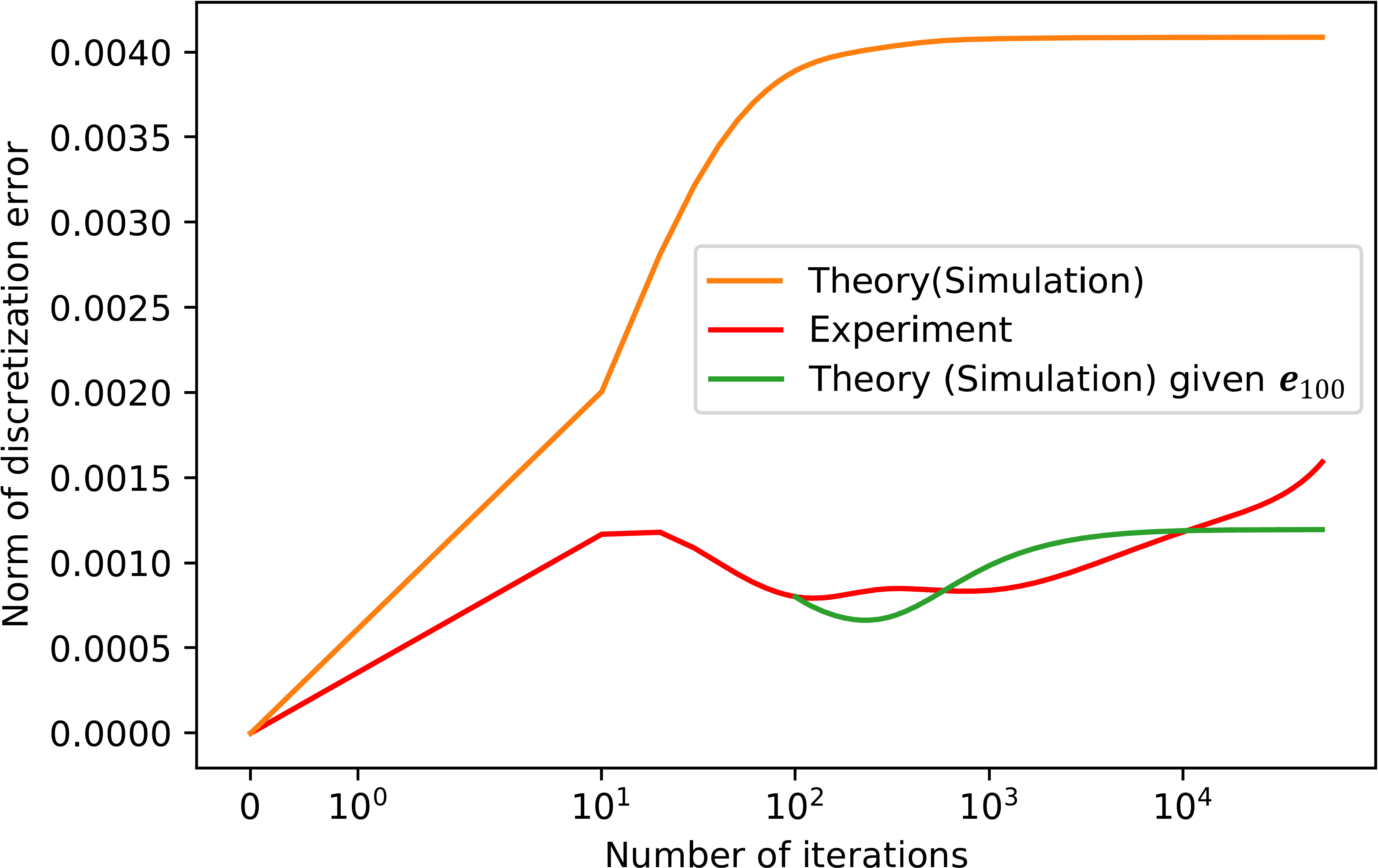}
        \subcaption{Learning rate = $10^{-3}$.}
    \end{minipage}
    \caption{\textbf{Theoretical prediction (Equation (\ref{eq: disc. err. with no xi})) vs. experimental result of discretization error between GF and GD.} The weight decay is $10^{-2}$. See Section \ref{sec: Experiment} and Appendix \ref{app: Details of Experiment} for experimental settings.}
    \label{fig: th. pred. vs. exp. of disc. err.}
\end{figure}

%% file: app_suppl_discussion.tex
\clearpage
\section{Supplementary Discussion} \label{app: Supplementary Discussion}
\paragraph{Supplementary related work (Section \ref{sec: Related Work}).}
To show the benefits of EoM, we focus on scale-invariant layers \cite{van2017l2_effectiveLR_orginal1, hoffer2018norm, zhang2018three, arora2018theoretical, chiley2019online, zhiyuan2020reconciling_NIPS2020, li2020exponential_learning_rate_ICLR2020, wan2021spherical_motion_dynamics, li2022robust, roburin2022spherical} and translation-invariant layers \cite{kunin2021neural_mechanics_1, tanaka2021noethers_neural_mechanics_2} in Section \ref{sec: Applications}.
To carry over the stability of a continuous optimization algorithm to a discretized system, the authors of \cite{bu2022feedback} add a feedback term to the optimization, and after that, they apply a discretization method to it. The authors' primary motivation is to keep the orthogonality of the weight parameters of DNNs, which is different from ours.

\paragraph{Convergence of $\bx$ (Section \ref{sec: Solution to EoLDE}).}
Note that the expansion of $\bx$ in terms of $\eta$ is not necessarily convergent, as is also pointed out in \cite{hairer2006geometric_numerical_integration_book}. Thus, we have to truncate the expansion at a suitable order. The discretization error at the truncation is given in Theorem \ref{cor: Discretization error of bx}.

\paragraph{Beyond leading order of discretization error (Theorem \ref{thm: Leading order of EoDE} and Section \ref{sec: Leading Order of Discretization Error Is Given By Counter Term}).}
In this work, we analyze the leading order of discretization error. However, higher-order terms cannot always be negligible.
We discuss in Section \ref{sec: Leading Order of Discretization Error Is Given By Counter Term} that the higher-order terms are important at the beginning of training. 


\paragraph{Existence of $\mcac$ (Section \ref{sec: Applications}).}
In our theoretical analysis of scale- and translation-invariant layers, the network contains both invariant ($\mca$) and non-invariant layers ($\mcac$), while previous works assume the whole network is invariant for simplicity \cite{van2017l2_effectiveLR_orginal1, hoffer2018norm, zhang2018three, arora2018theoretical, chiley2019online, zhiyuan2020reconciling_NIPS2020, li2020exponential_learning_rate_ICLR2020, wan2021spherical_motion_dynamics, li2022robust, roburin2022spherical}. We avoid this assumption and show that such mixed networks require appropriate modifications to analyses of invariant networks. For example, $\nabla f(\bth) = \f{1}{||\bth||} \nabla f (\hbth)$ for invariant networks, while $\nabla_\mca f (\bth) = \f{1}{||\btha||} \nabla_\mca f (\hbtha + \bthac)$ for mixed networks (Lemma \ref{lem: gradient relation of scale-invariant layers}), not $\f{1}{||\btha||} \nabla_\mca f (\hbtha)$. Such a naive replacement is not allowed.

\paragraph{Higher-order corrections to decay rate of $r_\mca$ (Section \ref{sec: Learning Dynamics of Scale-invariant Layers}).}
We can compute more corrections to the decay rate of $r_\mca$ ($\mca$ is a scale-invariant layer), using more counter terms. For example, a long algebra gives decay rate $\eta\la ( 1 + \f{\eta\la}{2} + \f{\eta^2\la^2}{3})$ for $\bx = \tbx_0 + \eta \tbx_1$. The proof is similar to Appendix \ref{app: Proof of thm: EoM of r}.

\paragraph{On equilibrium assumptions in Corollaries \ref{cor: r at equil.} and \ref{cor: angular update at equil.} (Section \ref{sec: Learning Dynamics of Scale-invariant Layers}).}
We make assumptions in Corollaries \ref{cor: r at equil.} and \ref{cor: angular update at equil.}; there exist two constants $r_{\mca *} \geq 0$ and $c_* \geq 0$ such that $r_\mca(t) \xrightarrow{t\rightarrow\infty} r_{\mca *}$ and $|| \nabla_\mca f(\hbth_\mca(t) + \bthac (t)) || \xrightarrow{t\rightarrow\infty} c_*$. These assumptions are similar to those given in previous studies \cite{van2017l2_effectiveLR_orginal1, wan2021spherical_motion_dynamics}. 
However, whether the assumptions are valid in the actual learning dynamics of DNNs is of independent interest.
In fact, the equilibrium assumption ($r_{\mca *}(t)$ and $||\nabla_\mca f(\hbtha(t) + \bthact)|| \xrightarrow{t\rightarrow \infty}$ constant) could not be satisfied even at one million steps of GD, and potentially because of it, $r_{\mca *}$ and $\De_*$ have a large discrepancy between the empirical results and theoretical predictions.
Deeper analyses on this point are needed. Under what conditions are the equilibrium assumptions valid? Can we relax the equilibrium assumptions and obtain realistic limiting dynamics of scale-invariant layers? This is exciting future work.

In contrast to our empirical result mentioned above, in \cite{wan2021spherical_motion_dynamics}, their experiments dramatically match their theoretical prediction.
This is potentially because of differences in experimental settings; in \cite{wan2021spherical_motion_dynamics}, SGD is used (ours is GD) and variance is induced, ResNet-50 \cite{ResNetV1, ResNetV2} is used (ours is a fully-connected network with three layers), ImageNet \cite{deng2009imagenet_original1, russakovsky2014imagenet_original2} and MSCOCO \cite{lin2014MSCOCO_original} are used (ours is MNIST \cite{MNIST}), and large learning rates ($\sim 10^{-1}$) and small weight decays ($\sim 10^{-4}$) are used (ours are given in Appendix \ref{app: Details of Experiment}).


\paragraph{Extension of EoM to general settings (Section \ref{sec: Conclusion and Limitations}).}
While we focus on GD and GF for simplicity, our counter-term-based approach and discretization error analysis can be extended to more general settings, such as SGD, acceleration methods (e.g., momentum SGD), and adaptive optimizers (e.g., Adam \cite{Adam}).
First, to extend our analysis to SGD, discretization error analysis of the Euler-Maruyama method, e.g., \cite{deng2016strong}, can be used. SDE's error analysis \cite{li2019stochastic_modified_equations_math_found_JMLR2019, feng2020uniform} is also relevant.
Second, we can extend our counter-term-based approach and discretization error analysis to acceleration methods by modifying the analysis for different differential equations from GF and different discretization schemes from the Euler method,
as is discussed in \cite{krichene2015accelerated, xu2018continuous_SDE_is_good_for_SGD, scieur2017integration}.
Third, \cite{barakat2021continuous_Adam} is the first work that provides a continuous approximation of Adam. However, its counter term and discretization error are open questions.